\documentclass[10pt,letterpaper,twocolumn,twoside]{IEEEtran}

\usepackage[utf8]{inputenc} 
\usepackage[T1]{fontenc}    
\usepackage[hyphens]{url}     
\usepackage{hyperref}       
\usepackage{booktabs}       
\usepackage{amsfonts}       
\usepackage{amssymb}        
\usepackage{nicefrac}       
\usepackage{microtype}      
\usepackage{bm}				
\usepackage{cite}			
\usepackage{graphicx}		
\usepackage{mathtools}		
\usepackage{algorithmic}	
\usepackage{algorithm}
\usepackage{amsthm}			
\usepackage{enumitem}		
\usepackage{multirow}       
\usepackage{tabularx}	    
\usepackage{hhline}
\usepackage{arydshln}		
\usepackage{xcolor}

\definecolor{lightgreen}{rgb}{.9,1,.9}

\newcolumntype{L}[1]{>{\raggedright\arraybackslash}p{#1}}
\newcolumntype{C}[1]{>{\centering\arraybackslash}p{#1}}
\newcolumntype{R}[1]{>{\raggedleft\arraybackslash}p{#1}}

\theoremstyle{plain} 
\newtheorem{proposition}{Proposition}
\newtheorem{definition}{Definition}
\newtheorem{theorem}{Theorem}

\newtheorem{assumption}{Assumption}

\def\defn{\,\coloneqq\,}
\def\argmin{\mathop{\mathsf{arg\,min}}} 
\def\argmax{\mathop{\mathsf{arg\,max}}}
\def\lim{\mathop{\mathsf{lim}}} 
\def\min{\mathop{\mathsf{min}}} 
\def\max{\mathop{\mathsf{max}}}

\def\prox{\mathsf{prox}}
\def\log{\mathsf{log}}
\def\zer{\mathsf{zer}}
\def\fix{\mathsf{fix}}
\def\dsf{\mathsf{\, d}}
\def\DnCNNast{{\text{DnCNN}^\ast}}

\def\Lmax{{L_{\mathsf{\tiny max}}}}

\def\ebm{{\bm{e}}}
\def\hbm{{\bm{h}}}

\def\xbm{{\bm{x}}}
\def\gbm{{\bm{g}}}
\def\ybm{{\bm{y}}}
\def\zbm{{\bm{z}}}
\def\rbm{{\bm{r}}}

\def\zerobm{\bm{0}}
\def\Abm{{\bm{A}}}
\def\Dbm{{\bm{D}}}
\def\Ibm{{\bm{I}}}

\def\xbmast{{\bm{x}^\ast}}

\def\xbmhat{{\widehat{\bm{x}}}}

\def\Tsf{{\mathsf{T}}}

\def\Dsf{{\mathsf{D}}}

\def\Hsf{{\mathsf{H}}}
\def\Nsf{{\mathsf{N}}}
\def\Gsf{{\mathsf{G}}}
\def\Isf{{\mathsf{I}}}
\def\Psf{{\mathsf{P}}}
\def\Rsf{{\mathsf{R}}}
\def\Usf{{\mathsf{U}}}
\def\Hsf{{\mathsf{H}}}

\def\R{\mathbb{R}}
\def\E{\mathbb{E}}

\def\Ncal{{\mathcal{N}}}

\def\Bcal{{\mathcal{B}}}

\begin{document}

\title{Block Coordinate Regularization by Denoising}

\author{Yu~Sun$^\ast$~\IEEEmembership{Student Member,~IEEE}, Jiaming~Liu$^\ast$~\IEEEmembership{Student Member,~IEEE}, \\
and Ulugbek~S.~Kamilov,~\IEEEmembership{Member,~IEEE}%
\thanks{This material is based upon work supported in part by NSF award CCF-1813910 and by NVIDIA Corporation with the donation of the Titan Xp GPU for research. This paper was presented at the 2019 33th Annual Conference on Neural Information Processing Systems (NeurIPS). }
\thanks{Y. Sun is with the Department of
Computer Science \& Engineering, Washington University in St.~Louis, MO 63130, USA.}
\thanks{J. Liu is with the Department of Electrical \& Systems Engineering, Washington University in St.~Louis, MO 63130, USA.}
\thanks{U.~S.~Kamilov (email:~kamilov@wustl.edu)
is with the Department of
Computer Science \& Engineering and the Department of Electrical \& Systems Engineering, Washington University in St.~Louis, MO 63130, USA.}
\thanks{$^\ast$ indicates equal contribution}}

\markboth{Block Coordinate Regularization by Denoising}%
{Sun,~Liu,~and~Kamilov}

\maketitle 

\ifCLASSOPTIONpeerreview
\begin{center} \bfseries EDICS: CIF-OBI, CIF-SBI, IMT-SIM\end{center}
\setcounter{page}{1}
\fi

\begin{abstract}
	We consider the problem of estimating a vector from its noisy measurements using a prior specified only through a denoising function. Recent work on plug-and-play priors (PnP) and regularization-by-denoising (RED) has shown the state-of-the-art performance of estimators under such priors in a range of imaging tasks. In this work, we develop a new block coordinate RED algorithm that decomposes a large-scale estimation problem into a sequence of updates over a small subset of the unknown variables. We theoretically analyze the convergence of the algorithm and discuss its relationship to the traditional proximal optimization. Our analysis complements and extends recent theoretical results for RED-based estimation methods. We numerically validate our method using several denoiser priors, including those based on convolutional neural network (CNN) denoisers.
\end{abstract}

\section{Introduction}

Problems involving estimation of an unknown vector $\xbm \in \R^n$ from a set of noisy measurements $\ybm \in \R^m$ are important in many areas, including computational imaging, machine learning, and compressive sensing. Consider the scenario in Fig.~\ref{Fig:Setting}, where a vector $\xbm \sim p_\xbm$ passes through the measurement channel $p_{\ybm | \xbm}$ to produce the measurement vector $\ybm$. When the estimation problem is ill-posed, it becomes essential to include the prior $p_\xbm$ in the estimation process. However, in high-dimensional settings, it is difficult to directly obtain the true prior $p_\xbm$ for certain signals  (such as natural images) and one is hence restricted to various indirect sources of prior information on $\xbm$. This paper considers the cases where the prior information on $\xbm$ is specified only via a denoising function, $\Dsf: \R^n \rightarrow \R^n$, designed for the removal of additive white Gaussian noise (AWGN).


\begin{figure}[t]
\centering\includegraphics[width=8.5cm]{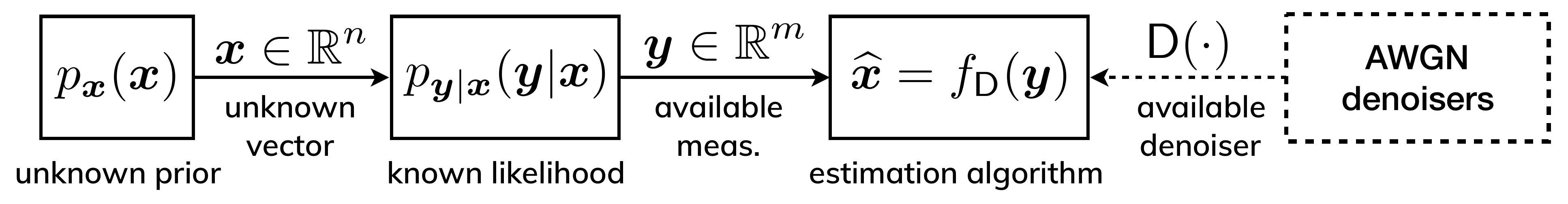}
\caption{The estimation problem considered in this work. The vector $\xbm \in \R^n$, with a prior $p_\xbm(\xbm)$, passes through the measurement channel $p_{\ybm|\xbm}(\ybm|\xbm)$ to result in the measurements $\ybm \in \R^m$. The estimation algorithm $f_\Dsf(\ybm)$ does not have a direct access to the prior, but can rely on a denoising function $\Dsf: \R^n \rightarrow \R^n$, specifically designed for the removal of additive white Gaussian noise (AWGN). We propose block coordinate RED as a scalable algorithm for obtaining $\xbm$ given $\ybm$ and $\Dsf$.}
\label{Fig:Setting}
\end{figure}


There has been considerable recent interest in leveraging denoisers as priors for the recovery of $\xbm$. One popular strategy, known as plug-and-play priors (PnP)~\cite{Venkatakrishnan.etal2013}, extends traditional proximal optimization~\cite{Parikh.Boyd2014} by replacing the proximal operator with a general off-the-shelf denoiser. It has been shown that the combination of proximal algorithms with advanced denoisers, such as BM3D~\cite{Dabov.etal2007} or DnCNN~\cite{Zhang.etal2017}, leads to the state-of-the-art performance for various imaging problems~\cite{Danielyan.etal2012, Chan.etal2016, Sreehari.etal2016, Ono2017, Kamilov.etal2017, Meinhardt.etal2017, Zhang.etal2017a, Buzzard.etal2017, Sun.etal2018a, Teodoro.etal2019, Ryu.etal2019}. A similar strategy has also been adopted in the context of a related class of algorithms known as approximate message passing (AMP)~\cite{Tan.etal2015, Metzler.etal2016, Metzler.etal2016a, Fletcher.etal2018}. Regularization-by-denoising (RED)~\cite{Romano.etal2017}, and the closely related deep mean-shift priors~\cite{Bigdeli.etal2017}, represent an alternative, in which the denoiser is used to specify an explicit regularizer that has a simple gradient. More recent work has clarified the existence of explicit RED regularizers~\cite{Reehorst.Schniter2019}, demonstrated its excellent performance on phase retrieval~\cite{Metzler.etal2018}, and further boosted its performance in combination with a deep image prior~\cite{Mataev.etal2019}. In short, the use of advanced denoisers has proven to be essential for achieving the state-of-the-art results in many contexts. However, solving the corresponding estimation problem is still a significant computational challenge, especially in the context of high-dimensional vectors $\xbm$, typical in modern applications.

In this work, we extend the  current family of RED algorithms by introducing a new \emph{block coordinate RED (BC-RED)} algorithm. The algorithm relies on random partial updates on $\xbm$, which makes it scalable to vectors that would otherwise be prohibitively large for direct processing. Additionally, as we shall see, the overall computational complexity of BC-RED can sometimes be lower than corresponding methods operating on the full vector. This behavior is consistent with the traditional coordinate descent methods that can outperform their full gradient counterparts by being able to better reuse local updates and take larger steps~\cite{Tseng2001, Nesterov2012, Beck.Tetruashvili2013, Wright2015, Fercoq.Gramfort2018}. We present two theoretical results related to BC-RED. We first theoretically characterize the convergence of the algorithm under a set of transparent assumptions on the data-fidelity and the denoiser. Our analysis complements the recent theoretical analysis of full-gradient RED algorithms in~\cite{Reehorst.Schniter2019} by considering block-coordinate updates and establishing the explicit worst-case convergence rate. Our second result establishes backward compatibility of BC-RED with the traditional proximal optimization. We show that when the denoiser corresponds to a proximal operator, BC-RED can be interpreted as an approximate MAP estimator, whose approximation error can be made arbitrarily small. To the best of our knowledge, this explicit link with proximal optimization is missing in the current literature on RED. BC-RED thus provides a flexible, scalable, and theoretically sound algorithm applicable to a wide variety of large-scale estimation problems. We demonstrate BC-RED on image recovery from linear measurements using several denoising priors, including those based on convolutional neural network (CNN) denoisers. A preliminary version of this work has appeared in~\cite{Sun.etal2019b}. The current paper contains all the proofs, more detailed descriptions and additional simulations.


\section{Background}

It is common to formulate the estimation in Figure~\ref{Fig:Setting} as an optimization problem
\begin{equation}
\label{Eq:RegMin}
\xbmhat = \argmin_{\xbm \in \R^n} f(\xbm) \quad\textrm{with}\quad f(\xbm) = g(\xbm) + h(\xbm),
\end{equation}
where $g$ is the data-fidelity term and $h$ is the regularizer. For example, the maximum a posteriori probability (MAP) estimator is obtained by setting
$$g(\xbm) = -\log(p_{\ybm|\xbm}(\ybm | \xbm)) \quad\textrm{and}\quad h(\xbm) = -\log(p_{\xbm}(\xbm)),$$
where $p_{\ybm|\xbm}$ is the likelihood that depends on $\ybm$ and $p_\xbm$ is the prior. One of the most popular data-fidelity terms is least-squares $g(\xbm) = \frac{1}{2}\|\ybm-\Abm\xbm\|_2^2$, which assumes a linear measurement model under AWGN. Similarly, one of the most popular regularizers is based on a sparsity-promoting penalty $h(\xbm) = \tau \|\Dbm\xbm\|_1$, where $\Dbm$ is a linear transform and $\tau > 0$ is the regularization parameter~\cite{Rudin.etal1992, Tibshirani1996, Candes.etal2006, Donoho2006}.

Many widely used regularizers, including the ones based on the $\ell_1$-norm, are nondifferentiable. Proximal algorithms~\cite{Parikh.Boyd2014}, such as the proximal-gradient method (PGM)~\cite{Figueiredo.Nowak2003, Daubechies.etal2004, Bect.etal2004, Beck.Teboulle2009} and alternating direction method of multipliers (ADMM)~\cite{Eckstein.Bertsekas1992, Afonso.etal2010, Ng.etal2010, Boyd.etal2011}, are a class of optimization methods that can circumvent the need to differentiate nonsmooth regularizers by using the proximal operator
\begin{equation}
\label{Eq:ProximalOperator}
\prox_{\mu h}(\zbm) \defn \argmin_{\xbm \in \R^n} \left\{\frac{1}{2}\|\xbm-\zbm\|_2^2 + \mu h(\xbm)\right\},\; \mu > 0.
\end{equation}
The observation that the proximal operator can be interpreted as the MAP denoiser for AWGN has prompted the development of PnP~\cite{Venkatakrishnan.etal2013}, where the proximal operator $\prox_{\mu h}(\cdot)$, within ADMM or PGM, is replaced with a more general denoising function $\Dsf(\cdot)$.

Consider the following alternative to PnP that also relies on a denoising function~\cite{Bigdeli.etal2017, Romano.etal2017}
\begin{align}
\label{Eq:REDGM}
&\xbm^t \leftarrow \xbm^{t-1} - \gamma \left(\nabla g(\xbm^{t-1})+\Hsf(\xbm^{t-1})\right) \nonumber\\ 
&\text{where}\quad \Hsf(\xbm) \defn \tau(\xbm - \Dsf(\xbm)), \quad \tau >0.
\end{align}
Under some conditions on the denoiser, it is possible to relate $\Hsf(\cdot)$ in~\eqref{Eq:REDGM} to some explicit regularization function $h$. For example, when the denoiser is locally homogeneous and has a symmetric Jacobian~\cite{Romano.etal2017, Reehorst.Schniter2019}, the operator $\Hsf(\cdot)$ corresponds to the gradient of the following function
\begin{equation}
\label{Eq:REDReg}
h(\xbm) = \frac{\tau}{2}\xbm^\Tsf(\xbm-\Dsf(\xbm)).
\end{equation}
On the other hand, when the denoiser corresponds to the minimum mean squared error (MMSE) estimator $\Dsf(\zbm) = \E[\xbm | \zbm]$ for the AWGN denoising problem~\cite{Bigdeli.etal2017, Reehorst.Schniter2019}, $\zbm = \xbm + \ebm$, with $\xbm \sim p_\xbm(\xbm)$ and $\ebm \sim \Ncal(\zerobm, \sigma^2\Ibm)$, the operator $\Hsf(\cdot)$ corresponds to the gradient of
\begin{equation}
\label{Eq:DMSP}
h(\xbm) = -\tau\sigma^2 \log(p_\zbm(\xbm)),
\end{equation}
where
\begin{equation}
p_\zbm(\xbm) = (p_\xbm \ast p_\ebm)(\xbm) = \int_{\R^n} p_\xbm(\zbm) \phi_\sigma(\xbm-\zbm) \dsf \zbm, \nonumber
\end{equation}
where $\phi_\sigma$ is the Gaussian probability density function of variance $\sigma^2$ and $\ast$ denotes convolution. In this paper, we will use the term RED to denote \emph{all} methods seeking the fixed points of~\eqref{Eq:REDGM}. The key benefits of the RED methods~\cite{Romano.etal2017, Bigdeli.etal2017, Metzler.etal2018, Reehorst.Schniter2019, Mataev.etal2019} are their explicit separation of the forward model from the prior, their ability to accommodate powerful denoisers (such as the ones based on CNNs) without differentiating them, and their state-of-the-art performance on a number of imaging tasks. The next section further extends the scalability of RED by designing a new block coordinate RED algorithm.

\section{Block Coordinate RED}
\label{Sec:Algorithm}

All the current RED algorithms operate on vectors in $\R^n$. We propose BC-RED, shown in Algorithm~\ref{Alg:BCRED}, to allow for partial randomized updates on $\xbm$. Consider the decomposition of $\R^n$ into $b \geq 1$ subspaces
$$\R^n = \R^{n_1} \times \R^{n_2} \times \cdots \times \R^{n_b}\quad\text{with}\quad n = n_1 + n_2 + \cdots + n_b.$$
For each $i \in \{1, \dots, b\}$, we define the matrix $\Usf_i: \R^{n_i} \rightarrow \R^n$ that injects a vector in $\R^{n_i}$ into $\R^n$ and its transpose $\Usf_i^\Tsf$ that extracts the $i$th block from a vector in $\R^n$. Then, for any $\xbm = (\xbm_1, \dots, \xbm_b) \in \R^n$
\begin{equation}
\label{Eq:SubspaceDecomposition}
\xbm = \sum_{i = 1}^b \Usf_i \xbm_i \quad\text{with}\quad \xbm_i = \Usf_i^\Tsf\xbm \in \R^{n_i}, \; i = 1, \dots, b
\end{equation}
which is equivalent to $\sum_{i = 1}^b \Usf_i \Usf_i^\Tsf = \Isf$.
Note that~\eqref{Eq:SubspaceDecomposition} directly implies the norm preservation
$\|\xbm\|_2^2 = \|\xbm_1\|_2^2 + \cdots + \|\xbm_b\|_2^2$ for any $\xbm \in \R^n$.
We are interested in a block-coordinate algorithm that uses only a subset of operator outputs corresponding to coordinates in some block $i \in \{1, \dots, b\}$. Hence, for an operator $\Gsf: \R^n \rightarrow \R^n$, we define the block-coordinate operator $\Gsf_i: \R^n \rightarrow \R^{n_i}$ as
\begin{equation}
\Gsf_i(\xbm) \defn [\Gsf(\xbm)]_i = \Usf_i^\Tsf\Gsf(\xbm) \in \R^{n_i}, \quad \xbm \in \R^n.
\end{equation}
We now introduce the proposed BC-RED algorithm summarized in Algorithm~\ref{Alg:BCRED}. Note that when $b = 1$, we have $n = n_1$ and $\Usf_1 = \Usf_1^\Tsf = \Isf$. Hence, the theoretical analysis in this paper is also applicable to the full-gradient RED algorithm in~\eqref{Eq:REDGM}.
\begin{algorithm}[t]
\caption{Block Coordinate Regularization by Denoising (BC-RED)}
\label{Alg:BCRED}
\begin{algorithmic}[1]
\STATE \textbf{input: } initial value $\xbm^0 \in \R^n$, parameter $\tau > 0$, and step-size $\gamma > 0$.
\FOR{$k = 1, 2, 3, \dots$}
\STATE Choose an index $i_k \in \{1, \dots, b\}$
\STATE $\xbm^k \leftarrow \xbm^{k-1} - \gamma \Usf_{i_k} \Gsf_{i_k}(\xbm^{k-1})$\\ 
\quad where  $\Gsf_i(\xbm) \defn \Usf_i^\Tsf\Gsf(\xbm)$ with $\Gsf(\xbm) \defn \nabla g(\xbm) + \tau(\xbm-\Dsf(\xbm))$.
\ENDFOR
\end{algorithmic}
\end{algorithm} 

As with traditional coordinate descent methods (see~\cite{Wright2015} for a review), BC-RED can be implemented using different block selection strategies. The strategy adopted for our theoretical analysis selects block indices $i_k$ as i.i.d.~random variables distributed uniformly over $\{1, \dots, b\}$. An alternative is to proceed in epochs of $b$ consecutive iterations, where at the start of each epoch the set $\{1, \dots, b\}$ is reshuffled, and $i_k$ is then selected consecutively from this ordered set. We numerically compare the convergence of both BC-RED variants in Section~\ref{Sec:Simulations}.

BC-RED updates its iterates one randomly picked block at a time using the output of $\Gsf$. When the algorithm converges, it converges to the vectors in the zero set of $\Gsf$
\begin{align}
&\Gsf(\xbmast) = \nabla g(\xbmast) + \tau(\xbmast - \Dsf(\xbmast)) = \zerobm \nonumber\\
&\quad\Leftrightarrow\quad \xbmast \in \zer(\Gsf) \defn \{\xbm \in \R^n : \Gsf(\xbm) = \zerobm\}.
\end{align}
Consider the following two sets
\begin{align}
&\zer(\nabla g) \defn \{\xbm \in \R^n : \nabla g(\xbm) = \zerobm\} \nonumber\\
\text{and}\quad& \fix(\Dsf) \defn \{\xbm \in \R^n : \xbm = \Dsf(\xbm)\},
\end{align}
where $\zer(\nabla g)$ is the set of all critical points of the data-fidelity and $\fix(\Dsf)$ is the set of all fixed points of the denoiser. Intuitively, the fixed points of $\Dsf$ correspond to all the vectors that are not denoised, and therefore can be interpreted as vectors that are \emph{noise-free} according to the denoiser. 

Note that if $\xbmast \in \zer(\nabla g)\cap\fix(\Dsf)$, then $\Gsf(\xbmast) = \zerobm$ and $\xbmast$ is one of the solutions of BC-RED. Hence, any vector that is consistent with the data for a convex $g$ and noiseless according to $\Dsf$ is in the solution set. On the other hand, when $\zer(\nabla g)\cap \fix(\Dsf) = \varnothing$, then $\xbmast \in \zer(\Gsf)$ corresponds to a tradeoff between the two sets, explicitly controlled via $\tau > 0$ (see Fig.~\ref{Fig:imageFlow} in the supplement for an illustration). This explicit control is one of the key differences between RED and PnP.

BC-RED benefits from considerable \emph{flexibility} compared to the full-gradient RED. Since each update is restricted to only one block of $\xbm$, the algorithm is suitable for parallel implementations and can deal with problems where the vector $\xbm$ is distributed in space and in time. However, the maximal benefit of BC-RED is achieved when $\Gsf_i$ is efficient to evaluate. Fortunately, it was systematically shown in~\cite{Peng.etal2016} that many operators---common in machine learning, image processing, and compressive sensing---admit \emph{coordinate friendly} updates. 

For a specific example, consider the least-squares data-fidelity $g$ and a block-wise denoiser $\Dsf$. Define the residual vector $r(\xbm) \defn \Abm\xbm-\ybm$ and consider a single iteration of BC-RED that produces $\xbm^+$ by updating the $i$th block of $\xbm$. Then, the update direction and the residual update can be computed as
\begin{align}
&\Gsf_i(\xbm) = \Abm_i^\Tsf r(\xbm) + \tau (\xbm_i - \Dsf(\xbm_i)) \nonumber\\
\text{and}\quad& r(\xbm^+) = r(\xbm) - \gamma \Abm_i \Gsf_i(\xbm),
\end{align}
where $\Abm_i \in \R^{m \times n_i}$ is a submatrix of $\Abm$ consisting of the columns corresponding to the $i$th block. In many problems of practical interest~\cite{Peng.etal2016}, the complexity of working with $\Abm_i$ is roughly $b$ times lower than with $\Abm$. Also, many advanced denoisers can be effectively applied on image patches rather than on the full image~\cite{Elad.Aharon2006, Buades.etal2010, Zoran.Weiss2011}. Therefore, in such settings, the speed of $b$ iterations of BC-RED is expected to be at least comparable to a single iteration of the full-gradient RED (see also Section~\ref{Sec:ComputationalComplexity}).

\section{Convergence Analysis and Compatibility with Proximal Optimization}
\label{Sec:TheoretcalResults}

In this section, we present two theoretical results related to BC-RED. We first establish its convergence to an element of $\zer(\Gsf)$ and then discuss its compatibility with the theory of proximal optimization.

\subsection{Fixed Point Convergence of BC-RED}

Our analysis requires three assumptions that together serve as sufficient conditions for convergence.

\begin{assumption}
\label{As:NonemptySet}
The operator $\Gsf$ is such that $\zer(\Gsf) \neq \varnothing$. There is a finite number $R_0$ such that the distance of the initial $\xbm^0 \in \R^n$ to the farthest element of $\zer(\Gsf)$ is bounded, that is
$$\max_{\xbmast \in \zer(\Gsf)} \|\xbm^0-\xbmast\|_2 \leq R_0.$$
\end{assumption}

This assumption is necessary to guarantee convergence and is related to the existence of the minimizers in the literature on traditional coordinate minimization~\cite{Tseng2001, Nesterov2012, Beck.Tetruashvili2013, Wright2015}.

The next two assumptions rely on Lipschitz constants along directions specified by specific blocks. We say that $\Gsf_i$ is \emph{block Lipschitz continuous} with constant $\lambda_i > 0$ if
\begin{align}
&\|\Gsf_i(\xbm)-\Gsf_i(\ybm)\|_2 \leq \lambda_i \|\hbm_i\|_2, \nonumber\\
\text{where}\quad& \xbm = \ybm + \Usf_i\hbm_i, \; \ybm \in \R^n, \hbm_i \in \R^{n_i}.
\end{align}
When $\lambda_i = 1$, we say that $\Gsf_i$ is \emph{block nonexpansive}. Note that if an operator $\Gsf$ is globally $\lambda$-Lipschitz continuous, then it is straightforward to see that each $\Gsf_i = \Usf_i^\Tsf\Gsf$ is also block $\lambda$-Lipschitz continuous.

\begin{assumption}
\label{As:DataFitConvexity}
The function $g$ is continuously differentiable and convex. Additionally, for each $i \in \{1, \dots, b\}$ the block gradient $\nabla_i g$ is block Lipschitz continuous with constant $L_i > 0$. We define the largest block Lipschitz constant as
$\Lmax \defn \max\{L_1, \dots, L_b\}.$
\end{assumption}

Let $L > 0$ denote the global Lipschitz constant of $\nabla g$. We always have $\Lmax \leq L$ and, for some $g$, it may even happen that $\Lmax = L/b$~\cite{Wright2015}. As we shall see, the largest possible step-size $\gamma$ of BC-RED depends on $\Lmax$, while that of the full-gradient RED on $L$. Hence, one natural advantage of BC-RED is that it can often take more aggressive steps compared to the full-gradient RED.

\begin{assumption}
\label{As:NonexpansiveDen}
The denoiser $\Dsf$ is such that each block denoiser $\Dsf_i$ is block nonexpansive.
\end{assumption}

Since the proximal operator is nonexpansive~\cite{Parikh.Boyd2014}, it automatically satisfies this assumption. We revisit this scenario in a greater depth in Section~\ref{Sec:ProximalConvergence}. We can now establish the following result for BC-RED.

\medskip
\begin{theorem}
\label{Thm:ConvThm1}
Run BC-RED for $t \geq 1$ iterations with random i.i.d.~block selection under Assumptions~\ref{As:NonemptySet}-\ref{As:NonexpansiveDen} using a fixed step-size $0 < \gamma \leq 1/(\Lmax+2\tau)$. Then, we have
\begin{align}
\label{Eq:BCREDConv}
\E &\left[\min_{k \in \{1, \dots, t\}} \|\Gsf(\xbm^{k-1})\|_2^2\right] \nonumber\\
&\leq \E\left[\frac{1}{t}\sum_{k = 1}^t \|\Gsf(\xbm^{k-1})\|_2^2\right]
\leq \frac{b(\Lmax+2\tau)}{\gamma t}R_0^2.
\end{align}
\end{theorem}

A proof of the theorem is provided in the supplement. Theorem~\ref{Thm:ConvThm1} establishes the fixed-point convergence of BC-RED in expectation to $\zer(\Gsf)$ with $O(1/t)$ rate. The proof relies on the monotone operator theory~\cite{Bauschke.Combettes2017, Ryu.Boyd2016}, widely used in the context of convex optimization~\cite{Parikh.Boyd2014}, including in the unified analysis of various traditional coordinate descent algorithms~\cite{Peng.etal2016a, Chow.etal2017}. Note that the theorem does \emph{not} assume the existence of any regularizer $h$, which makes it applicable to denoisers beyond those characterized with explicit functions in~\eqref{Eq:REDReg} and~\eqref{Eq:DMSP}.

Since $\Lmax \leq L$, one important implication of Theorem~\ref{Thm:ConvThm1}, is that the worst-case convergence rate (in expectation) of $b$ iterations of BC-RED is better than that of a single iteration of the full-gradient RED (to see this, note that the full-gradient rate is obtained by setting $b = 1$, $\Lmax = L$, and removing the expectation in~\eqref{Eq:BCREDConv}). This implies that in \emph{coordinate friendly settings} (as discussed at the end of Section~\ref{Sec:Algorithm}), the overall computational complexity of BC-RED can be lower than that of the full-gradient RED. This gain is primarily due to two factors: (a) possibility to pick a larger step-size $\gamma = 1/(\Lmax+2\tau)$; (b) immediate reuse of each local block-update when computing the next iterate (the full-gradient RED updates the full vector before computing the next iterate).

In the special case of $\Dsf(\xbm) = \xbm - (1/\tau)\nabla h(\xbm)$, for some convex function $h$, BC-RED reduces to the traditional coordinate descent method applied to~\eqref{Eq:RegMin}. Hence, under the assumptions of Theorem~\ref{Thm:ConvThm1}, one can rely on the analysis of traditional randomized coordinate descent methods in~\cite{Wright2015} to obtain
\begin{equation}
\label{Eq:CordDesConv}
\E\left[f(\xbm^t)\right] - f^\ast \leq \frac{2b}{\gamma t}R_0^2
\end{equation}
where $f^\ast$ is the minimum value in~\eqref{Eq:RegMin}. A proof of~\eqref{Eq:CordDesConv} is provided in the supplement for completeness. Therefore, such denoisers lead to explicit convex RED regularizers and $O(1/t)$ convergence of BC-RED in terms of the objective. However, as discussed in Section~\ref{Sec:ProximalConvergence}, when the denoiser is a proximal operator of some convex $h$, BC-RED is \emph{not} directly solving~\eqref{Eq:RegMin}, but rather its approximation.

Finally, note that the analysis in Theorem~\ref{Thm:ConvThm1} only provides \emph{sufficient conditions} for the convergence of BC-RED. As corroborated by our numerical studies in Section~\ref{Sec:Simulations}, the actual convergence of BC-RED is more general and often holds beyond nonexpansive denoisers. One plausible explanation for this is that such denoisers are \emph{locally nonexpansive} over the set of input vectors used in testing. On the other hand, the recent techniques for spectral-normalization of CNNs~\cite{Miyato.etal2018, Sedghi.etal2019, Gouk.etal2018} provide a convenient tool for building \emph{globally nonexpansive} neural denoisers that result in provable convergence of BC-RED.

\subsection{Convergence for Proximal Operators}
\label{Sec:ProximalConvergence}
One of the limitations of the current RED theory is in its limited backward compatibility with the theory of proximal optimization.  For example, as discussed in~\cite{Romano.etal2017} (see section \emph{``Can we mimic any prior?''}), the popular total variation (TV) denoiser~\cite{Rudin.etal1992} cannot be justified with the original RED regularization function~\eqref{Eq:REDReg}. In this section, we show that BC-RED (and hence also the full-gradient RED) can be used to solve~\eqref{Eq:RegMin} for any convex, closed, and proper function $h$. We do this by establishing a formal link between RED and the concept of Moreau smoothing, widely used in nonsmooth optimization~\cite{Moreau1965, Rockafellar.Wets1998, Yu2013}. In particular, we consider the following proximal-operator denoiser
\begin{align}
\label{Eq:ProximalDenoiser}
\Dsf(\zbm) = \prox_{\frac{1}{\tau} h}(\zbm) = \argmin_{\xbm \in \R^n}\left\{\frac{1}{2}\|\xbm-\zbm\|_2^2 + (1/\tau) h(\xbm)\right\},
\end{align}
where $\tau>0$, $\zbm\in\R^n$, and $h$ is a closed, proper, and convex function~\cite{Parikh.Boyd2014}. Since the proximal operator is nonexpansive, it is also block nonexpansive, which means that Assumption~\ref{As:NonexpansiveDen} is automatically satisfied. Our analysis, however, requires an additional assumption using the constant $R_0$ defined in Assumption~\ref{As:NonemptySet}. 

\begin{assumption}
\label{As:Subgradient}
There is a finite number $G_0$ that bounds the largest subgradient of $h$, that is
$$\max\{\|\gbm(\xbm)\|_2 : \gbm(\xbm) \in \partial h(\xbm), \xbm \in \Bcal(\xbm^0, R_0)\} \leq G_0,$$
where $\Bcal(\xbm^0, R_0) \defn \{\xbm \in \R^n : \|\xbm-\xbm^0\|_2 \leq R_0\}$ denotes a ball of radius $R_0$, centered at $\xbm^0$.
\end{assumption}

This assumption on boundedness of the subgradients holds for a large number of regularizers used in practice, including both TV and the $\ell_1$-norm penalties. We can now establish the following result.

\medskip
\begin{theorem}
\label{Thm:ProxConv}
Run BC-RED for $t \geq 1$ iterations with random i.i.d.~block selection and the denoiser~\eqref{Eq:ProximalDenoiser} under Assumptions~\ref{As:NonemptySet}-\ref{As:Subgradient} using a fixed step-size $0 < \gamma \leq 1/(\Lmax+2\tau)$. Then, we have
\begin{equation}
\label{Eq:BCREDProx}
\E\left[f(\xbm^t)\right] - f^\ast \leq \frac{2b}{\gamma t} R_0^2 + \frac{G_0^2}{2\tau},
\end{equation}
where the function $f$ is defined in~\eqref{Eq:RegMin} and $f^\ast$ is its minimum.
\end{theorem}

The theorem is proved in the supplement. It establishes that BC-RED in expectation \emph{approximates} the solution of~\eqref{Eq:RegMin} with an error bounded by $(G_0^2/(2\tau))$. For example, by setting $\tau = \sqrt{t}$ and ${\gamma = 1/(\Lmax+2\sqrt{t})}$, one obtains the following bound
\begin{equation}
\label{Eq:BCREDProx2}
\E\left[f(\xbm^t)\right] - f^\ast \leq \frac{1}{\sqrt{t}}\left[2b(\Lmax+2)R_0^2 + G_0^2\right].
\end{equation}

When $h(\xbm) = -\log(p_\xbm(\xbm))$, the proximal operator corresponds to the MAP denoiser, and the solution of BC-RED corresponds to an \emph{approximate} MAP estimator. This approximation can be made as precise as desired by considering larger values for the parameter $\tau > 0$. Note that this further justifies the RED framework by establishing that it can be used to compute a minimizer of any proper, closed, and convex (but not necessarily differentiable) $h$. Therefore, our analysis strengthens RED by showing that it can accommodate a much larger class of explicit regularization functions, beyond those characterized in~\eqref{Eq:REDReg} and~\eqref{Eq:DMSP}.

\section{Numerical Validation}
\label{Sec:Simulations}

There is a considerable recent interest in using advanced priors in the context of image recovery from underdetermined ($m < n$) and noisy measurements. Recent work~\cite{Romano.etal2017, Bigdeli.etal2017, Reehorst.Schniter2019, Metzler.etal2018, Mataev.etal2019} suggests significant performance improvements due to advanced denoisers (such as BM3D~\cite{Dabov.etal2007} or DnCNN~\cite{Zhang.etal2017}) over traditional sparsity-driven priors (such as TV~\cite{Rudin.etal1992}). Our goal is to complement these studies with several simulations validating our theoretical analysis and providing additional insights into {BC-RED}. The code for our implementation of BC-RED is available through the following link\footnote{\url{https://github.com/wustl-cig/bcred}}.

We consider inverse problems of form $\ybm = \Abm\xbm + \ebm,$
where ${\ebm \in \R^m}$ is an AWGN vector and ${\Abm \in \R^{m \times n}}$ is a matrix corresponding to either a sparse-view Radon transform, i.i.d.~zero-mean Gaussian random matrix of variance $1/m$, or radially subsampled two-dimensional Fourier transform. Such matrices are commonly used in the context of computerized tomography (CT)~\cite{Kak.Slaney1988}, compressive sensing~\cite{Candes.etal2006, Donoho2006}, and magnetic resonance imaging (MRI)~\cite{Knoll.etal2011}, respectively. In all simulations, we set the measurement ratio to be approximately $m/n = 0.5$ with AWGN corresponding to input signal-to-noise ratio (SNR) of 30 dB and 40 dB. The images used correspond to 10 images randomly selected from the NYU fastMRI dataset~\cite{Zbontar.etal2018}, resized to be $160 \times 160$ pixels (see Fig.~\ref{Fig:TestImages} in the supplement). BC-RED is set to work with 16 blocks, each of size $40 \times 40$ pixels. The reconstruction quality is quantified using SNR averaged over all ten test images.

In addition to well-studied denoisers, such as TV and BM3D, we design our own CNN denoiser denoted $\DnCNNast$, which is a simplified version of the popular DnCNN denoiser (see Supplement~\ref{Sec:TechnicalDetails} for details). This simplification reduces the computational complexity of denoising, which is important when running many iterations of BC-RED. Additionally, it makes it easier to control the global Lipschitz constant of the CNN via spectral-normalization~\cite{Sedghi.etal2019}. We train $\DnCNNast$ for the removal of AWGN at four noise levels corresponding to $\sigma \in \{5,10,15,20\}$. For each experiment, we select the denoiser achieving the highest SNR value. Note that the $\sigma$ parameter of BM3D is also fine-tuned for each experiment from the same set $\{5,10,15,20\}$.


\begin{figure*}[t]
\centering\includegraphics[width=0.8\textwidth]{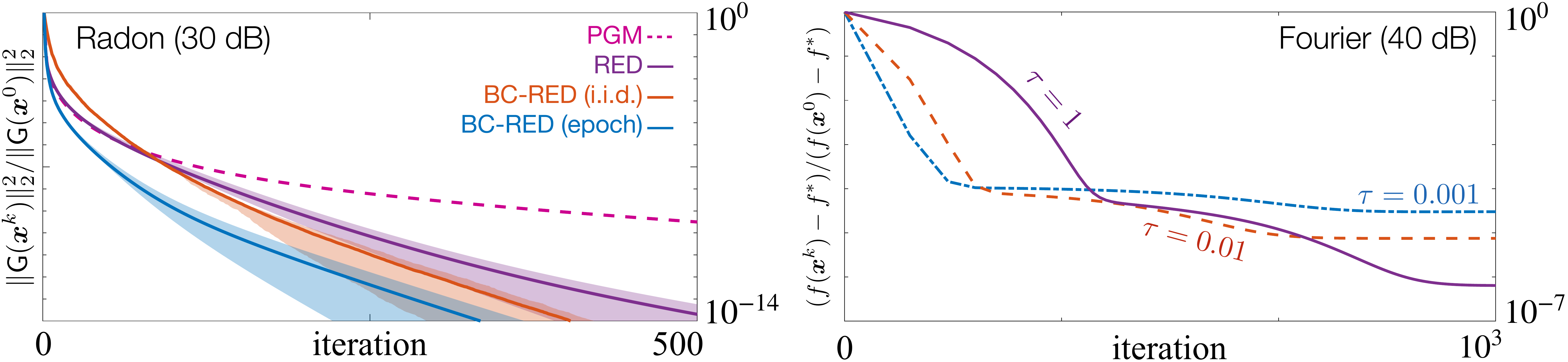}
\caption{\textbf{Left}: Illustration of the convergence of BC-RED under a nonexpansive $\DnCNNast$ prior. Average normalized distance to $\zer(\Gsf)$ is plotted against the iteration number with the shaded areas representing the range of values attained over all test images. \textbf{Right}: Illustration of the influence of the parameter $\tau > 0$ for solving TV regularized least-squares problem using BC-RED. As $\tau$ increases, BC-RED provides an increasingly accurate approximation to the TV optimization problem.}
\label{Fig:convergenceCT}
\end{figure*}


\begin{table*}[t]
\centering
\caption{Average SNRs obtained for different measurement matrices and image priors.}
\vspace{5pt}
\small
\label{Tab:SNR}
\begin{tabular*}{13.3cm}{L{90pt}C{30pt}C{30pt}cC{30pt}C{30pt}cC{30pt}C{30pt}} 
\toprule
\textbf{Methods} & \multicolumn{2}{c}{\textbf{Radon}} & & \multicolumn{2}{c}{\textbf{Random}} & & \multicolumn{2}{c}{$\textbf{Fourier}$} \\
\midrule
& \textbf{30 dB} & \textbf{40 dB} & & \textbf{30 dB} & \textbf{40 dB} & & \textbf{30 dB} & \textbf{40 dB} \\
\cmidrule{2-3} \cmidrule{5-6} \cmidrule{8-9}
\textbf{PGM (TV)} 	& 20.66 & 24.40 & & 26.07 & \emph{\textbf{28.42}} & & 28.74 & 29.99 \\
\textbf{U-Net} 		& \emph{\textbf{21.90}} & 21.72 & & 16.37 & 16.40 & & 22.11 & 22.11 \\
\noalign{\vskip 1pt}\cdashline{1-9}\noalign{\vskip 3pt}
\textbf{RED (TV)} 	& 20.79 & 24.46 & & 25.64 & \colorbox{lightgreen}{\makebox(24,6){28.30}} & & 28.67 & 29.97 \\
\textbf{BC-RED (TV)} & 20.78 & 24.42 & & 25.70 & \colorbox{lightgreen}{\makebox(24,6){28.39}} &  & 28.71 & 29.99 \\
\noalign{\vskip 1pt}\cdashline{1-9}\noalign{\vskip 3pt}
\textbf{RED (BM3D)} & \colorbox{lightgreen}{\makebox(24,6){21.55}} & \colorbox{lightgreen}{\makebox(24,6){\emph{\textbf{25.24}}}} & & 26.46 & 27.82 & & 28.89 & 29.79 \\
\textbf{BC-RED (BM3D)} & \colorbox{lightgreen}{\makebox(24,6){21.56}} & \colorbox{lightgreen}{\makebox(24,6){25.16}} & & 26.50 & 27.88 & & 28.85 & 29.80 \\
\noalign{\vskip 1pt}\cdashline{1-9}\noalign{\vskip 3pt}
\textbf{RED ($\text{DnCNN}^\ast$)} & 20.89 & 24.38 & & \colorbox{lightgreen}{\makebox(24,6){26.53}} & 28.05 & & \colorbox{lightgreen}{\makebox(24,6){29.33}} & \colorbox{lightgreen}{\makebox(24,6){\emph{30.32}}} \\
\textbf{BC-RED ($\text{DnCNN}^\ast$)} & 20.88 & 24.42 & & \colorbox{lightgreen}{\makebox(24,6){\emph{\textbf{26.60}}}} & 28.12 & & \colorbox{lightgreen}{\makebox(24,6){\emph{\textbf{29.40}}}} & \colorbox{lightgreen}{\makebox(24,6){\emph{\textbf{30.39}}}} \\
\bottomrule
\end{tabular*}
\end{table*}

Theorem~\ref{Thm:ConvThm1} establishes the convergence of BC-RED in expectation to an element of $\zer(\Gsf)$. This is illustrated in Fig.~\ref{Fig:convergenceCT} (left) for the Radon matrix with $30$ dB noise and a nonexpansive $\DnCNNast$ denoiser (see also Fig.~\ref{Fig:ConvergencePlots} in the supplement). The average value of $\|\Gsf(\xbm^k)\|_2^2/\|\Gsf(\xbm^0)\|_2^2$ is plotted against the iteration number for the full-gradient RED and BC-RED, with $b$ updates of BC-RED (each modifying a single block) represented as one iteration. We numerically tested two block selection rules for BC-RED (\emph{i.i.d.} and \emph{epoch}) and observed that processing in randomized epochs leads to a faster convergence. For reference, the figure also plots the normalized squared norm of the gradient mapping vectors produced by the traditional PGM with TV~\cite{Beck.Teboulle2009a}. The shaded areas indicate the range of values taken over $10$ runs corresponding to each test image. The results highlight the potential of BC-RED to enjoy a better convergence rate compared to the full-gradient RED, with BC-RED (epoch) achieving the accuracy of $10^{-10}$ in 104 iterations, while the full-gradient RED achieves the same accuracy in 190 iterations.

Theorem \ref{Thm:ProxConv} establishes that for proximal-operator denoisers, BC-RED computes an approximate solution to~\eqref{Eq:RegMin} with an accuracy controlled by the parameter $\tau$. This is illustrated in Fig.~\ref{Fig:convergenceCT} (right) for the Fourier matrix with $40$ dB noise and the TV regularized least-squares problem. The average value of ${(f(\xbm^k)-f^\ast)/(f(\xbm^0)-f^\ast)}$ is plotted against the iteration number for BC-RED with ${\tau \in \{0.01,0.1, 1\}}$. The optimal value $f^\ast$ is obtained by running the traditional PGM until convergence. As before, the figure groups $b$ updates of BC-RED as a single iteration. The results are consistent with our theoretical analysis and show that as $\tau$ increases BC-RED provides an increasingly accurate solution to TV. On the other hand, since the range of possible values for the step-size $\gamma$ depends on $\tau$, the speed of convergence to $f^\ast$ is also influenced by $\tau$.

The benefits of the full-gradient RED algorithms have been well discussed in prior work~\cite{Romano.etal2017, Bigdeli.etal2017, Reehorst.Schniter2019, Metzler.etal2018, Mataev.etal2019}. Table~\ref{Tab:SNR} summarizes the average SNR performance of BC-RED in comparison to the full-gradient RED for all three matrix types and several priors. Unlike the full-gradient RED, BC-RED is implemented using block-wise denoisers that work on image patches rather than the full images. We empirically found that 40 pixel padding on the denoiser input is sufficient for BC-RED to match the performance of the full-gradient RED. The table also includes the results for the traditional PGM with TV~\cite{Beck.Teboulle2009a} and the widely-used end-to-end U-Net approach~\cite{Jin.etal2017a, Han.etal2017}. The latter first backprojects the measurements into the image domain and then denoises the result using U-Net~\cite{Ronneberger.etal2015}. The model was specifically trained end-to-end for the Radon matrix with 30 dB noise and applied as such to other measurement settings. All the algorithms were run until convergence with hyperparameters optimized for SNR. The $\DnCNNast$ denoiser in the table corresponds to the residual network with the Lipschitz constant of two (see Supplement~\ref{Sec:ArchitectureTraining} for details). The overall best SNR in the table is highlighted in bold-italic, while the best RED prior is highlighted in light-green. First, note the excellent agreement between BC-RED and the full-gradient RED. This close agreement between two methods is encouraging as BC-RED relies on block-wise denoising and our analysis does not establish uniqueness of the solution, yet, in practice, both methods seem to yield solutions of nearly identical quality. Second, note that BC-RED and RED provide excellent approximations to PGM-TV solutions. Third, note how (unlike U-Net) BC-RED and RED with $\DnCNNast$ generalize to different measurement models. Finally. no prior seems to be universally good on all measurement settings, which indicates to the potential benefit of tailoring specific priors to specific measurement models.



\begin{figure*}[t]
\centering\includegraphics[width=0.9\textwidth]{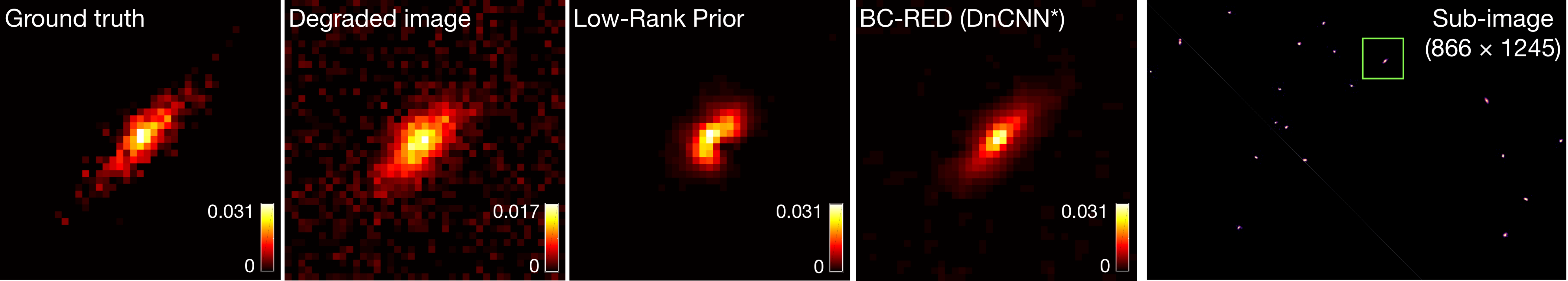}
\caption{Recovery of a $8292 \times 8364$ pixel galaxy image degraded by a spatially variant blur and a high-amount of AWGN. The efficacy of BC-RED is due to the natural sparsity in this large-scale problem, with all of the information contained in a small part of the full image.}
\label{Fig:galaxyImages}
\end{figure*}


Coordinate descent methods are known to be highly beneficial in problems where both $m$ and $n$ are very large, but each measurement depends only on a small subset of the unknowns~\cite{Niu.etal2011}. Fig.~\ref{Fig:galaxyImages} demonstrates BC-RED in such large-scale setting by adopting the experimental setup from a recent work~\cite{Farrens.etal2017} (see also Fig.~\ref{Fig:MoreGalaxies} in the supplement). Specifically, we consider the recovery of a $8292 \times 8364$ pixel galaxy image degraded by 597 known point spread functions (PSFs) corresponding to different spatial locations. The natural sparsity of the problem makes it ideal for BC-RED, which is implemented to update $41 \times 41$ pixel blocks in a randomized fashion by only picking areas containing galaxies. The computational complexity of BC-RED is further reduced by considering a simpler variant of $\DnCNNast$ that has only four convolutional layers (see Fig.~\ref{Fig:DnCNNstar} in the supplement). For comparison, we additionally show the result obtained by using the low-rank recovery method from~\cite{Farrens.etal2017} with all the parameters kept at the values set by the authors. Note that our intent here is not to justify $\DnCNNast$ as a prior for image deblurring, but to demonstrate that BC-RED can indeed be applied to a realistic, nontrivial image recovery task on a large image.


\section{Conclusion and Future Work}
\label{Sec:Conclusion}

Coordinate descent methods have become increasingly important in optimization for solving large-scale problems arising in data analysis. We have introduced BC-RED as a coordinate descent extension to the current family of RED algorithms and theoretically analyzed its convergence. Preliminary experiments suggest that BC-RED can be an effective tool in large-scale estimation problems arising in image recovery. More experiments are certainly needed to better asses the promise of this approach in various estimation tasks. For future work, we would like to explore accelerated and asynchronous variants of BC-RED to further enhance its performance in parallel settings.


\appendix

We adopt the monotone operator theory~\cite{Ryu.Boyd2016, Bauschke.Combettes2017} for a unified analysis of BC-RED. In Supplement~\ref{Sec:Proof1}, we prove the convergence of BC-RED to an element of $\zer(\Gsf)$. In Supplement~\ref{Sec:Proof2}, we prove that for proximal-operator denoisers, BC-RED converges to an approximate solution of~\eqref{Eq:RegMin}. For completeness, in Supplement~\ref{Sec:CoordinateAnalysis}, we discuss the well-known convergence results for traditional coordinate descent~\cite{Tseng2001, Nesterov2012, Beck.Tetruashvili2013, Wright2015, Fercoq.Gramfort2018}. In Supplement~\ref{Sec:BackgroundMaterial}, we provide the background material used in Supplement~\ref{Sec:Proof1} and Supplement~\ref{Sec:Proof2}, expressed in a form convenient for block-coordinate analysis. In Supplement~\ref{Sec:TechnicalDetails}, we provide additional technical details omitted from the main paper due to space, such as the details on computational complexity and CNN architectures. In Supplement~\ref{Sec:AdditionalSimulations}, we present additional simulations that were also omitted from the main paper due to space.

\section{Proof of Theorem~\ref{Thm:ConvThm1}}
\label{Sec:Proof1}

A fixed-point convergence of averaged operators is well-known under the name of Krasnosel'skii-Mann theorem (see Section 5.2 in~\cite{Bauschke.Combettes2017}) and was recently applied to the analysis of PnP~\cite{Sun.etal2018a} and several full-gradient RED algorithms in~\cite{Reehorst.Schniter2019}. Our analysis here extends these results to the block-coordinate setting and provides explicit worst-case convergence rates for BC-RED.

We consider the following operators
$$\Gsf_i = \nabla_i g + \Hsf_i \quad\text{with}\quad \Hsf_i = \tau \Usf_i^\Tsf(\Isf-\Dsf).$$
and proceed in several steps.
\begin{enumerate}[label=(\alph*), leftmargin=*]
\item Since $\nabla_i g$ is block $L_i$-Lipschitz continuous, it is also block $\Lmax$-Lipschitz continuous. Hence, we know from Proposition~\ref{Prop:BlockCocoer} in Supplement~\ref{Sec:Convexity} that it is block $(1/\Lmax)$-cocoercive. Then from Proposition~\ref{Prop:NonexpEquiv} in Supplement~\ref{Sec:AveragedOp}, we know that the operator ${(\Usf_i^\Tsf-(2/\Lmax)\nabla_i g)}$ is block nonexpansive.
\item From the definition of $\Hsf_i$ and the fact that $\Dsf_i$ is block nonexpansive, we know that ${(\Usf_i^\Tsf-(1/\tau)\Hsf_i) = \Dsf_i}$ is block nonexpansive.

\item From Proposition~\ref{Prop:BlockConvNonexp} in Supplement~\ref{Sec:Prelims}, we know that a convex combination of block nonexpansive operators is also block nonexpansive, hence we conclude that
\begin{align*}
&\Usf_i^\Tsf - \frac{2}{\Lmax+2\tau}\Gsf_i \\
&= \left(\frac{2}{\Lmax+2\tau}\cdot \frac{\Lmax}{2}\right)\left[\Usf_i^\Tsf-\frac{2}{\Lmax}\nabla_i g\right] \nonumber\\ 
&\quad\quad\quad+ \left(\frac{2}{\Lmax+2\tau}\cdot \frac{2\tau}{2}\right)\left[\Usf_i^\Tsf - \frac{1}{\tau}\Hsf_i\right],
\end{align*}
is block nonexpansive. Then from Proposition~\ref{Prop:NonexpEquiv} in Supplement~\ref{Sec:AveragedOp}, we know that $\Gsf_i$ is block $1/(\Lmax+2\tau)$-cocoercive.
\item Consider any $\xbmast \in \zer(\Gsf)$, an index $i \in \{1, \dots, b\}$ picked uniformly at random, and a single iteration of BC-RED $\xbm^+ = \xbm - \gamma\Usf_i\Gsf_i\xbm$. Define a vector $\hbm_i \defn \Usf_i^\Tsf(\xbm-\xbmast) \in \R^{n_i}$. We then have
\begin{align}
\label{Eq:SingleIter}
\nonumber&\|\xbm^+-\xbmast\|^2 \\
\nonumber=& \|\xbm - \xbmast - \gamma \Usf_i\Gsf_i\xbm\|^2\\
\nonumber=& \|\xbm-\xbmast\|^2 - 2\gamma (\Usf_i\Gsf_i\xbm)^\Tsf(\xbm-\xbmast) + \gamma^2\|\Gsf_i\xbm\|^2 \\
\nonumber=& \|\xbm-\xbmast\|^2 - 2\gamma (\Gsf_i\xbm-\Gsf_i\xbmast)^\Tsf\hbm_i + \gamma^2\|\Gsf_i\xbm\|^2 \\
\nonumber\leq& \|\xbm-\xbmast\|^2 - \frac{2\gamma-(\Lmax+2\tau)\gamma^2}{\Lmax+2\tau}\|\Gsf_i\xbm\|^2 \\
\leq& \|\xbm-\xbmast\|^2 - \frac{\gamma}{\Lmax+2\tau}\|\Gsf_i \xbm\|^2,
\end{align}
where in the third line we used $\Gsf_i\xbmast = \Usf_i^\Tsf\Gsf\xbmast = \zerobm$, in the fourth line the block cocoercivity of $\Gsf_i$, and in the last line the fact that $0 < \gamma \leq 1/(\Lmax+2\tau)$. 

\item By taking a conditional expectation on both sides and rearranging the terms, we obtain
\begin{align*}
&\frac{\gamma}{\Lmax + 2\tau} \E\left[\|\Gsf_i\xbm\|^2 | \xbm\right] \nonumber\\
&=  \frac{\gamma}{b(\Lmax + 2\tau)} \sum_{i = 1}^b \|\Gsf_i\xbm\|^2 = \frac{\gamma}{b(\Lmax + 2\tau)} \|\Gsf\xbm\|^2 \nonumber\\ 
&\leq   \E\left[\|\xbm-\xbmast\|^2 - \|\xbm^+ - \xbmast\|^2 | \xbm \right]
\end{align*}
\item Hence by averaging over $t \geq 1$ iterations and taking the total expectation
\begin{align}
\E\left[\frac{1}{t}\sum_{k = 1}^{t} \|\Gsf\xbm^{k-1}\|^2\right] &\leq \frac{1}{t}\left[\frac{b(\Lmax+2\tau)}{\gamma}\|\xbm^0-\xbmast\|^2\right] \nonumber\\
&\leq \frac{1}{t}\left[\frac{b(\Lmax+2\tau)}{\gamma}R_0^2\right].
\end{align}
\end{enumerate}
The last inequality directly leads to the result.

\textbf{Remark}. Eq.~\eqref{Eq:SingleIter} implies that, under Assumptions~\ref{As:NonemptySet}-\ref{As:NonexpansiveDen}, the iterates of BC-RED satisfy
\begin{equation}
\label{Eq:DistanceReduction}
\|\xbm^t-\xbmast\| \leq \|\xbm^{t-1}-\xbmast\| \leq \cdots \leq \|\xbm^0-\xbmast\| \leq R_0,
\end{equation}
which means that the distance of the iterates of BC-RED to $\zer(\Gsf)$ is nonincreasing.

\textbf{Remark}. Suppose we are solving a \emph{coordinate friendly problem}~\cite{Peng.etal2016}, in which the cost of the full gradient update is $b$ times the cost of block update. Consider the step-size $\gamma = 1/(L + 2\tau)$ where $L$ is the global Lipschitz constant of the gradient method. A similar analysis as above would yield the following convergence rate for the gradient method
$$\frac{1}{t}\sum_{k = 1}^{t} \|\Gsf\xbm^{k-1}\|^2 \leq \frac{(L+2\tau)^2R_0^2}{t}$$
Now, consider the step-size $\gamma = 1/(\Lmax + 2\tau)$ and suppose that we run $(t \cdot b)$ updates of BC-RED with $t \geq 1$. Then, we have that
$$\E\left[\frac{1}{tb}\sum_{k = 1}^{tb} \|\Gsf\xbm^{k-1}\|^2\right] \leq \frac{(\Lmax + 2\tau)^2 R_0^2}{t}.$$
Since $\Lmax  \leq L \leq b\Lmax$, where the upper bound can sometimes be tight, we conclude that the expected complexity of the block-coordinate algorithm is lower compared to the full algorithm.

\section{Proof of Theorem~\ref{Thm:ProxConv}}
\label{Sec:Proof2}

The concept of Moreau smoothing is well-known and has been extensively used in other contexts (see for example~\cite{Yu2013}). Our contribution is to formally connect the concept to RED-based algorithms, which leads to its novel justification as an approximate MAP estimator. The basic review of relevant concepts from proximal optimization is given in Supplement~\ref{Sec:MoreauTheory}.

For $\tau > 0$, we consider the Moreau envelope of $h$
$$h_{(1/\tau)}(\xbm) \defn \min_{\zbm \in \R^n}\left\{\frac{1}{2}\|\zbm-\xbm\|^2 + (1/\tau) h(\zbm)\right\}.$$
From Proposition~\ref{Prop:UniformBoundMoreau} in~Supplement~\ref{Sec:MoreauTheory} we know that
\begin{equation}
\label{Eq:MorApprox}
0 \leq h(\xbm) - \tau h_{(1/\tau)}(\xbm) \leq \frac{G_0}{2\tau}
\end{equation}
and from Proposition~\ref{Prop:GradMorProxRes} in~Supplement~\ref{Sec:MoreauTheory}, we know that
\begin{equation}
\label{Eq:MorGrad}
\tau\nabla h_{(1/\tau)}(\xbm) = \tau(\xbm - \prox_{(1/\tau)h}(\xbm)).\end{equation}
Hence, we can express the function $f$ as follows
\begin{align*}
f(\xbm) 
&= g(\xbm) + h(\xbm) \\
&= (g(\xbm) + \tau h_{(1/\tau)}(\xbm)) + (h(\xbm) - \tau h_{(1/\tau)}(\xbm))  \\
&= f_{(1/\tau)}(\xbm) + (h(\xbm) - \tau h_{(1/\tau)}(\xbm)),
\end{align*}
where $f_{(1/\tau)} \defn g + \tau h_{(1/\tau)}$.
From eq.~\eqref{Eq:MorGrad}, we conclude that a single iteration of BC-RED
$$\xbm^+ = \xbm - \gamma \Usf_i \Gsf_i \xbm \quad\text{with}\quad \Gsf_i = \Usf_i^\Tsf(\nabla g(\xbm) + \tau \nabla h_{(1/\tau)}(\xbm))$$
 is performing a block-coordinate descent on the function $f_{(1/\tau)}$.
From eq.~\eqref{Eq:MorApprox} and the convexity of the Moreau envelope, we have
$$f_{(1/\tau)}^\ast = f_{(1/\tau)}(\xbmast) \leq f_{(1/\tau)}(\xbm) \leq f(\xbm), \quad \xbm \in \R^n, \xbmast \in \zer(\Gsf).$$
Hence, there exists a finite $f^\ast$ such that $f(\xbm) \geq f^\ast$ with $f_{(1/\tau)}^\ast \leq f^\ast$. Consider the iteration $t \geq 1$ of BC-RED, then we have that
\begin{align*}
\E[f(\xbm^t)] - f^\ast 
&\leq \E[f(\xbm^t)] - f_{(1/\tau)}^\ast \\
&= (\E[f_{(1/\tau)}(\xbm^t)]-f_{(1/\tau)}^\ast) \nonumber\\
&\quad\quad\quad\quad+ \E[(h(\xbm^t)-\tau h_{(1/\tau)}(\xbm^t)]) \\
&\leq \frac{2b}{\gamma t}R_0^2 + \frac{G_0^2}{2\tau},
\end{align*}
where we applied~\eqref{Eq:CordDesConv}, which is further discussed in Supplement~\ref{Sec:CoordinateAnalysis}.

The proof of eq.~\eqref{Eq:BCREDProx2} is directly obtained by setting $\tau = \sqrt{t}$, $\gamma = \Lmax+2\sqrt{t}$, and noting that $t \geq \sqrt{t}$, for all $t \geq 1$.

\section{Convergence of the Traditional Coordinate Descent}
\label{Sec:CoordinateAnalysis}

The following analysis has been adopted from~\cite{Wright2015}. We include it here for completeness.

Consider the following denoiser 
$$\Dsf(\xbm) = \xbm - \frac{1}{\tau}\nabla h(\xbm), \quad \tau > 0, \quad \xbm \in \R^n,$$ 
and the following function
$$f(\xbm) = g(\xbm) + h(\xbm)$$
where $g$ and $h$ are both convex and continuously differentiable. For this denoiser, we have that
$$\Gsf(\xbm) = \nabla g(\xbm) + \tau (\xbm-\Dsf(\xbm)) = \nabla g(\xbm) + \nabla h(\xbm) = \nabla f(\xbm).$$
Therefore, in this case, BC-RED is minimizing a convex and smooth function $f$, which means that any $\xbmast \in \zer(\Gsf)$ is a global minimizer of $f$. Additionally, due to Proposition~\ref{Prop:NonexpCocoerOp} in Supplement~\ref{Sec:Prelims} and Proposition~\ref{Prop:BlockCocoer} in Supplement~\ref{Sec:Convexity}, we have
\begin{align}
&\Dsf_i \text{ is block nonexpansive} \nonumber\\
\Leftrightarrow\quad &\nabla_i h \text{ is block $2\tau$-Lipschitz continuous}.
\end{align}
Hence, for such denoisers, Assumption~\ref{As:NonexpansiveDen} is equivalent to the $2\tau$-Lipschitz smoothness of block gradients $\nabla_i h$.

To prove eq.~\ref{Eq:CordDesConv}, we consider the following iteration
$$\xbm^+ = \xbm - \Usf_i \Gsf_i\xbm \quad\text{with}\quad \Gsf_i = \nabla_i f = \nabla_i g + \nabla_i h,$$
which under our assumptions is a special case of the setting for Theorem~\ref{Thm:ConvThm1}.
\begin{enumerate}[label=(\alph*), leftmargin=*]
\item From the block Lipscthiz continuity of $f$, we conclude that
\begin{align*}
f(\xbm^+) 
&\leq f(\xbm) + \nabla f(\xbm)^\Tsf(\xbm^+-\xbm) \nonumber\\
&\quad\quad\quad\quad\quad\quad\quad\quad+\frac{(\Lmax+2\tau)}{2}\|\xbm^+-\xbm\|^2 \\
&= f(\xbm) - \gamma \|\nabla_i f(\xbm)\|^2 \nonumber\\
&\quad\quad\quad\quad\quad\quad\quad\quad+ \frac{\gamma^2(\Lmax+2\tau)}{2}\|\nabla_i f(\xbm)\|^2 \\
&\leq f(\xbm) - \frac{\gamma}{2} \|\nabla_i f(\xbm)\|^2,
\end{align*}
where the last inequality comes from the fact that $\gamma \leq 1/(\Lmax+2\tau)$. 

\item For all $t \geq 1$, define
$$\varphi_t \defn \E\left[f(\xbm^t)\right] - f(\xbmast).$$
Then from (a), we can conclude that
\begin{align}
\varphi_t &\leq \varphi_{t-1} - \frac{\gamma}{2b}\E\left[\|\nabla f(\xbm^{t-1})\|^2\right] \nonumber\\
&\leq \varphi_{t-1} - \frac{\gamma}{2b}\E\left[\|\nabla f(\xbm^{t-1})\|\right]^2,
\end{align}
where in the last inequality we used the Jensen's inequality, and the fact that
\begin{align}
\E\left[\|\nabla_i f(\xbm^{t-1})\|^2\right] &= \E\left[\E\left[\|\nabla_i f(\xbm^{t-1})\|^2 | \xbm^{t-1}\right]\right] \nonumber\\
&= \E\left[\frac{1}{b} \sum_{i = 1}^b \|\nabla_i f(\xbm^t)\|^2 \right] \nonumber\\
&= \frac{1}{b}\E\left[\|\nabla f(\xbm^{t-1})\|^2\right].
\end{align}

\item From convexity, we know that
\begin{align}
\varphi_t = \E\left[f(\xbm^t)\right] - f(\xbmast) &\leq \E\left[\nabla f(\xbm^t)^\Tsf(\xbm^t-\xbmast)\right] \nonumber\\
&\leq \E\left[\|\nabla f(\xbm^t)\| \|\xbm^t-\xbmast\|\right] \nonumber\\
&\leq R_0 \cdot \E\left[\|\nabla f(\xbm^t)\|\right],
\end{align}
where in the last inequality, we used eq.~\eqref{Eq:DistanceReduction}. This combined with the result of (b) implies that
$$\varphi_t \leq \varphi_{t-1} - \frac{\gamma}{2b} \frac{\varphi_{t-1}^2}{R_0^2}.$$

\item Note that from (c), we can obtain
$$\frac{1}{\varphi_t}-\frac{1}{\varphi_{t-1}} = \frac{\varphi_{t-1}-\varphi_t}{\varphi_t\varphi_{t-1}} \geq \frac{\varphi_{t-1}-\varphi_t}{\varphi_{t-1}^2} \geq \frac{\gamma}{2b R_0^2}.$$
By iterating this inequality, we get the final result
$$\frac{1}{\varphi_t} \geq \frac{1}{\varphi_0} + \frac{\gamma t}{2b\|\xbm^0-\xbmast\|^2} \geq \frac{\gamma t}{2bR_0^2} \;\Rightarrow\; \varphi_t \leq \frac{2b}{\gamma t}R_0^2.$$
\end{enumerate}

\section{Background Material}
\label{Sec:BackgroundMaterial}

The results in this section are well-known in the optimization literature and can be found in different forms in standard textbooks~\cite{Rockafellar.Wets1998, Boyd.Vandenberghe2004, Nesterov2004, Bauschke.Combettes2017}. For completeness, we summarize the key results useful for our analysis by restating them in a block-coordinate form.

\subsection{Properties of Block-Coordinate Operators}
\label{Sec:Prelims}

Most of the concepts in this part come from the traditional monotone operator theory~\cite{Ryu.Boyd2016, Bauschke.Combettes2017} adapted for block-coordinate operators.

\begin{definition}
We define the \emph{block-coordinate operator}  $\Tsf_i: \R^n \rightarrow \R^{n_i}$ of $\Tsf: \R^n \rightarrow \R^n$ as 
$$\Tsf_i\xbm \defn [\Tsf\xbm]_i = \Usf_i^\Tsf \Tsf\xbm \in \R^{n_i}, \quad \xbm \in \R^n.$$
The operator $\Tsf_i$ applies $\Tsf$ to its input vector and then extracts the subset of outputs corresponding to the coordinates in the block $i \in \{1, \dots, b\}$.
\end{definition}

\textbf{Remark}.~When $b = 1$, we have that $n = n_1$ and $\Usf_1 = \Usf_1^\Tsf = \Isf$. Then, all the properties in this section reduce to their standard counterparts from the monotone operator theory in $\R^n$. In such settings, we simply drop the word \emph{block} from the name of the property.

\medskip
\begin{definition}
$\Tsf_i$ is \emph{block Lipschitz continuous with constant $\lambda_i > 0$} if
$$\|\Tsf_i\xbm - \Tsf_i\ybm\| \leq \lambda_i\|\hbm_i\|,\quad  \xbm = \ybm + \Usf_i\hbm_i, \quad \ybm \in \R^n, \hbm_i \in \R^{n_i}.$$
When $\lambda_i = 1$, we say that $\Tsf_i$ is \emph{block nonexpansive}. 
\end{definition}

\begin{definition}
An operator $\Tsf_i$ is \emph{block cocoercive with constant $\beta_i > 0$} if
$$(\Tsf_i\xbm-\Tsf_i\ybm)^\Tsf\hbm_i \geq \beta_i\|\Tsf_i\xbm-\Tsf_i\ybm\|^2,$$
$$ \xbm = \ybm + \Usf_i\hbm_i, \quad \ybm \in \R^n, \hbm_i \in \R^{n_i}.$$
When $\beta_i = 1$, we say that $\Tsf_i$ is \emph{block firmly nonexpansive}.
\end{definition}

\medskip\noindent
The following propositions are conclusions derived from the definition of above.

\medskip
\begin{proposition}
\label{Prop:BlockConvNonexp}
Let $\Tsf_{ij}: \R^n \rightarrow \R^{n_i}$ for $j \in J$ be a set of block nonexpansive operators. Then, their convex combination
$$\Tsf_i \defn \sum_{j \in J} \theta_j \Tsf_{ij}, \quad\text{with}\quad \theta_j > 0 \text{ and } \sum_{j \in J} \theta_j = 1,$$
is nonexpansive.
\end{proposition}

\begin{proof}
By using the triangular inequality and the definition of block nonexpansiveness, we obtain
$$\|\Tsf_i\xbm-\Tsf_i\ybm\| \leq \sum_{j \in J} \theta_j \|\Tsf_{ij}\xbm-\Tsf_{ij}\ybm\| \leq \left(\sum_{j \in J}\theta_j\right) \|\hbm_i\| = \|\hbm_i\|,$$
for all $\ybm \in \R^n$ and $\hbm_i \in \R^{n_i}$ where $\xbm = \ybm + \Usf_i\hbm_i$ .
\end{proof}

\begin{proposition}
\label{Prop:NonexpCocoerOp}
Consider $\Rsf_i = \Usf_i^\Tsf - \Tsf_i$ where $\Tsf_i: \R^n \rightarrow \R^{n_i}$.
$$\Tsf_i \text{ is block nonexpansive } \quad\Leftrightarrow\quad \Rsf_i \text{ is $(1/2)$-block cocoercive.}$$
\end{proposition}

\begin{proof}
First suppose that $\Rsf_i$ is $1/2$ block cocoercive. Let $\xbm = \ybm + \Usf_i\hbm_i$ for all $\ybm \in \R^n$ and $\hbm_i \in \R^{n_i}$. We then have
$$\frac{1}{2}\|\Rsf_i\xbm-\Rsf_i\ybm\|^2 \leq (\Rsf_i\xbm-\Rsf_i\ybm)^\Tsf\hbm_i = \|\hbm_i\|^2 - (\Tsf_i\xbm-\Tsf_i\ybm)^\Tsf\hbm_i.$$
We also have that
$$\frac{1}{2}\|\Rsf_i\xbm-\Rsf_i\ybm\|^2 = \frac{1}{2}\|\hbm_i\|^2 - (\Tsf_i\xbm-\Tsf_i\ybm)^\Tsf\hbm_i + \frac{1}{2}\|\Tsf_i\xbm-\Tsf_i\ybm\|^2.$$
By combining these two and simplifying the expression, we obtain that 
$$\|\Tsf_i\xbm-\Tsf_i\ybm\| \leq \|\hbm_i\|.$$
The converse can be proved by following this logic in reverse.
\end{proof}

\subsection{Block Averaged Operators}
\label{Sec:AveragedOp}

It is well known that the iteration of a nonexpansive operator does not necessarily converge. To see this consider a nonexpansive operator $\Tsf = -\Isf$, where $\Isf$ is identity. However, it is also well known that the convergence can be established for averaged operators.

\begin{definition}
For a constant $\alpha \in (0, 1)$, we say that the operator $\Tsf$ is \emph{$\alpha$-averaged}, if there exists a nonexpansive operator $\Nsf$ such that $\Tsf = (1-\alpha)\Isf + \alpha \Nsf$.
\end{definition}

\begin{definition}
For a constant $\alpha \in (0, 1)$, we say that $\Tsf_i: \R^n \rightarrow \R^{n_i}$ is \emph{block $\alpha$-averaged}, if there exists a block nonexpansive operator $\Nsf_i$ such that $\Tsf_i = (1-\alpha)\Usf_i^\Tsf + \alpha \Nsf_i$.
\end{definition}

\textbf{Remark}.~It is clear that if $\Tsf$ is $\alpha$-averaged, then $\Tsf_i = \Usf_i^\Tsf\Tsf$ is block $\alpha$-averaged.

The following characterization is often convenient.
\begin{proposition}
\label{Prop:BlockAveragedEquiv}
For a block nonexpansive operator $\Tsf_i$, a constant $\alpha \in (0, 1)$, and the operator ${\Rsf_i \defn \Usf_i^\Tsf-\Tsf_i}$, the following are equivalent
\begin{enumerate}[label=(\alph*), leftmargin=*]
\item $\Tsf_i$ is block $\alpha$-averaged
\item $(1-1/\alpha)\Usf_i^\Tsf + (1/\alpha)\Tsf_i$ is block nonexpansive
\item $\|\Tsf_i\xbm - \Tsf_i\ybm\|^2 \leq \|\hbm_i\|^2 - \left(\frac{1-\alpha}{\alpha}\right)\|\Rsf_i\xbm-\Rsf_i\ybm\|^2$, $\hspace{2cm}\xbm = \ybm + \Usf_i\hbm_i, \quad \ybm \in \R^n, \hbm_i \in \R^{n_i}$
\end{enumerate}
\end{proposition}

\begin{proof}
The equivalence of (a) and (b) is clear from the definition. To establish the equivalence with (c), consider an operator $\Nsf_i$ and $\Tsf_i = (1-\alpha)\Usf_i^\Tsf + \alpha \Nsf_i$. Note that
$$\Rsf_i = \Usf_i^\Tsf - \Tsf_i = \alpha (\Usf_i^\Tsf - \Nsf_i).$$
Then, for all $\ybm \in \R^n$ and $\hbm_i \in \R^{n_i}$, with $\xbm = \ybm + \Usf_i \hbm_i$, we have that
\begin{align}
\label{Equ:AvgExpansion1}
\|\Tsf_i\xbm - \Tsf_i\ybm\|^2
\nonumber&= \|(1-\alpha)\hbm_i + \alpha (\Nsf_i\xbm-\Nsf_i\ybm)\|^2\\
\nonumber&= (1-\alpha) \|\hbm_i\|^2 + \alpha \|\Nsf_i\xbm-\Nsf_i\ybm\|^2 - \nonumber\\
&\hspace{1.6cm}\alpha(1-\alpha)\|\hbm_i - (\Nsf_i\xbm-\Nsf_i\ybm)\|^2 \nonumber\\
&= (1-\alpha) \|\hbm_i\|^2 + \alpha \|\Nsf_i\xbm-\Nsf_i\ybm\|^2 \nonumber\\
&\hspace{1.6cm}- \left(\frac{1-\alpha}{\alpha}\right)\|\Rsf_i\xbm-\Rsf_i\ybm\|^2,
\end{align}
where we used the fact that 
$$\|(1-\alpha) \xbm + \alpha \ybm\|^2 = (1-\alpha) \|\xbm\|^2 + \alpha \|\ybm\|^2 - \alpha(1-\alpha)\|\xbm-\ybm\|^2,$$
where $\theta \in \R$ and $\xbm, \ybm \in \R^n$. Consider also
\begin{align}
\label{Equ:AvgExpansion2}
&\|\hbm_i\|^2 - \left(\frac{1-\alpha}{\alpha}\right)\|\Rsf_i\xbm-\Rsf_i\ybm\|^2 \nonumber\\
&= (1-\alpha)\|\hbm_i\|^2 + \alpha \|\hbm_i\|^2 - \left(\frac{1-\alpha}{\alpha}\right)\|\Rsf_i\xbm-\Rsf_i\ybm\|^2.
\end{align}
It is clear that we have
\begin{align}
\eqref{Equ:AvgExpansion1} \leq \eqref{Equ:AvgExpansion2} &\quad\Leftrightarrow\quad \Nsf_i \text{ is block nonexpansive} \nonumber\\
&\quad\Leftrightarrow\quad \Tsf_i \text{ is block $\alpha$-averaged},
\end{align}
where for the last equivalence, we used the definition of block averagedness.
\end{proof}

\begin{proposition}
\label{Prop:NonexpEquiv}
Consider a block-coordinate operator ${\Tsf_i = \Usf_i^\Tsf\Tsf}$ with $\Tsf: \R^n \rightarrow \R^n$. Let $\xbm = \ybm + \Usf_i\hbm$ with ${\xbm \in \R^n}$, ${\hbm_i \in \R^{n_i}}$ and consider $\beta_i > 0$. Then, the following are equivalent
\begin{enumerate}[label=(\alph*), leftmargin=*]
\item $\Tsf_i$ is block $\beta_i$-cocoercive
\item $\beta_i\Tsf_i$ is block firmly nonexpansive
\item $\Usf_i^\Tsf-\beta_i\Tsf_i$ is block firmly nonexpansive.
\item $\beta_i\Tsf_i$ is block $(1/2)$-averaged.
\item $\Usf_i^\Tsf-2\beta_i\Tsf_i$ is block nonexpansive.
\end{enumerate}
\end{proposition}

\begin{proof}
The equivalence between (a) and (b) is readily observed by defining $\Psf_i \defn \beta_i\Tsf_i$ and noting that
\begin{align}
(\Psf_i\xbm - \Psf_i\ybm)^\Tsf\hbm_i = \beta_i(\Tsf_i\xbm - \Tsf_i\ybm)^\Tsf\hbm_i \nonumber\\
\quad\text{and}\quad \|\Psf_i\xbm-\Psf_i\ybm\|^2 = \beta_i^2 \|\Tsf_i\xbm-\Tsf_i\ybm\|.
\end{align}

\medskip\noindent
Define $\Rsf_i \defn \Usf_i^\Tsf - \Psf_i$ and suppose (b) is true, then
\begin{align*}
(\Rsf_i\xbm-\Rsf_i\ybm)^\Tsf\hbm_i 
&= \|\hbm_i\|^2 - (\Psf_i\xbm-\Psf_i\ybm)^\Tsf\hbm_i \\
&= \|\Rsf_i\xbm-\Rsf_i\ybm\|^2 + (\Psf_i\xbm-\Psf_i\ybm)^\Tsf\hbm_i \nonumber\\
&\hspace{3cm}- \|\Psf_i\xbm-\Psf_i\ybm\|^2 \\
&\geq \|\Rsf_i\xbm-\Rsf_i\ybm\|^2.
\end{align*}
By repeating the same argument for $\Psf_i = \Usf_i^\Tsf - \Rsf_i$, we establish the full equivalence between (b) and (c).

\medskip\noindent
The full equivalence of (b) and (d) can be established by observing that
\begin{align*}
&\hspace{-2.9cm}2\|\Psf_i\xbm-\Psf_i\ybm\|^2 \leq 2(\Psf_i\xbm-\Psf_i\ybm)^\Tsf\hbm_i \\
\Leftrightarrow\quad\|\Psf_i\xbm-\Psf_i\ybm\|^2 &\leq 2(\Psf_i\xbm-\Psf_i\ybm)^\Tsf\hbm_i - \|\Psf_i\xbm-\Psf_i\ybm\|^2 \\
&= \|\hbm_i\|^2-(\|\hbm_i\|^2 - 2(\Psf_i\xbm-\Psf_i\ybm)^\Tsf\hbm_i\nonumber\\
&\hspace{3cm} + \|\Psf_i\xbm-\Psf_i\ybm\|^2)\\
&= \|\hbm_i\|^2 - \|\Rsf_i\xbm-\Rsf_i\ybm\|^2.
\end{align*}

\medskip\noindent
To show the equivalence with (e), first suppose that ${\Nsf_i \defn \Usf_i^\Tsf - 2 \Psf_i}$ is block nonexpansive, then ${\Psf_i = \frac{1}{2}(\Usf_i^\Tsf + (-\Nsf_i))}$ is block $1/2$-averaged, which means that it is block firmly nonexpansive. On the other hand, if $\Psf_i$ is block firmly nonexpansive, then it is block $1/2$-averaged, which means that from Proposition~\ref{Prop:BlockAveragedEquiv}(b) we have that $(1-2)\Usf_i^\Tsf + 2\Psf_i = 2\Psf_i - \Usf_i^\Tsf = -\Nsf_i$ is block nonexpansive. This directly means that $\Nsf_i$ is block nonexpansive.
\end{proof}

\subsection{Operator Properties for Convex Function}
\label{Sec:Convexity}

It is convenient to link properties of a function $f: \R^n \rightarrow \R$, $\xbm \mapsto y = f(\xbm)$, to the properties of operators derived from it. The key properties for our analysis are related to continuity and convexity.

\begin{proposition}
Let $f$ be continuously differentiable function with $\nabla_i f$ that is block $L_i$-Lipschitz continuous. Then,
\begin{align}
f(\ybm) &\leq f(\xbm) + \nabla f(\xbm)^\Tsf(\ybm-\xbm) + \frac{L_i}{2}\|\ybm-\xbm\|^2 \nonumber\\
&= f(\xbm) + \nabla_i f(\xbm)^\Tsf\hbm_i + \frac{L_i}{2}\|\hbm_i\|^2 \nonumber
\end{align}
for all $\xbm \in \R^n$ and $\hbm_i \in \R^{n_i}$, where $\ybm = \xbm + \Usf_i\hbm_i$.
\end{proposition}

\begin{proof}
The proof is a minor variation of the one presented in Section~2.1 of~\cite{Nesterov2004}.
\end{proof}

\begin{proposition}
\label{Prop:LipBound2}
Consider a continuously differentiable $f$ such that $\nabla_i f$ is block $L_i$-Lipschitz continuous. Let $\xbmast \in \R^n$ denote the global minimizer of $f$. Then, we have that
\begin{align}
\frac{1}{2L_i} \|\nabla_i f(\xbm)\|^2 \leq (f(\xbm)-f(\xbmast)) \leq \frac{L_i}{2}\|\xbm-\xbmast\|^2, \nonumber\\
\text{where}\quad \xbm = \xbmast + \Usf_i \hbm_i, \quad \xbm \in \R^n, \hbm_i \in \R^{n_i}. \nonumber
\end{align}
\end{proposition}

\begin{proof}
The proof is a minor variation of the discussion in Section~9.1.2 of~\cite{Boyd.Vandenberghe2004}. 
\end{proof}

\begin{proposition}
\label{Prop:BlockCocoer}
For a convex and continuously differentiable function $f$, we have
\begin{align}
\nabla_i f \text{ is block $L_i$-Lipschitz continuous} \nonumber\\
\Leftrightarrow\quad \nabla_i f \text{ is block $(1/L_i)$-cocoercive}. \nonumber
\end{align}
\end{proposition}

\begin{proof}
The proof is a minor variation of the one presented as Theorem~2.1.5 in Section~2.1 of~\cite{Nesterov2004}.
\end{proof}

\subsection{Moreau smoothing and proximal operators}
\label{Sec:MoreauTheory}

In this section, we consider a class of functions that are proper, closed, and convex, but are not necessarily differentiable. The proximal operator is a widely-used concept in such nonsmooth optimization problems~\cite{Moreau1965, Rockafellar.Wets1998}.

\begin{definition}
\label{Def:MoreauEnv}
Consider a proper, closed, and convex $h$ and a constant $\mu > 0$. We define the \emph{proximal operator}
$$\prox_{\mu h}(\xbm) \defn \argmin_{\zbm \in \R^n}\left\{\frac{1}{2}\|\zbm-\xbm\|^2 + \mu h(\zbm)\right\}$$
and the \emph{Moreau envelope}
$$h_\mu(\xbm) \defn \min_{\zbm \in \R^n} \left\{\frac{1}{2}\|\zbm-\xbm\|^2 + \mu h(\zbm)\right\}.$$
\end{definition}

\begin{proposition}
\label{Prop:GradMorProxRes}
The function $h_\mu$ is convex and continuously differentiable with a $1$-Lipschitz gradient
$$\nabla h_\mu(\xbm) = \xbm - \prox_{\mu h}(\xbm), \quad \xbm \in \R^n.$$
\end{proposition}

\begin{proof}
We first show that $h_\mu$ is convex. Consider
$$q(\xbm, \zbm) \defn \frac{1}{2}\|\zbm-\xbm\|^2 + \mu h(\zbm),$$
which is convex $(\xbm, \zbm)$. Then, for any $0 \leq \theta \leq 1$ and $(\xbm_1, \zbm_1), (\xbm_2, \zbm_2) \in \R^{2n}$, we have 
\begin{align}
h_\mu(\theta\xbm_1+(1-\theta)\xbm_2) &\leq q(\theta\xbm_1+(1-\theta)\xbm_2, \theta\zbm_1+(1-\theta)\zbm_2) \nonumber\\
&\leq \theta q(\xbm_1, \zbm_1) + (1-\theta) q(\xbm_2, \zbm_2),
\end{align}
where we used the convexity of $q$. Since this inequality holds everywhere, we have
$$h_\mu(\theta\xbm_1+(1-\theta)\xbm_2) \leq \theta h_\mu (\xbm_1) + (1-\theta) h_\mu(\xbm_2),$$ 
with $$h_\mu(\xbm_1) = \min_{\zbm_1}q(\xbm_1, \zbm_1) \quad\text{and}\quad h_\mu(\xbm_2) = \min_{\zbm_2}q(\xbm_2, \zbm_2).$$

\medskip\noindent
To show the differentiability, note that
\begin{align*}
h_\mu(\xbm) 
&= \frac{1}{2}\|\xbm\|^2 - \max_{\zbm \in \R^n}\left\{\xbm^\Tsf\zbm - \mu h(\zbm) - \frac{1}{2}\|\zbm\|^2\right\} \\
&= \frac{1}{2}\|\xbm\|^2 - \phi^\star(\xbm) \quad\text{with}\quad \phi(\zbm) \defn \frac{1}{2}\|\zbm\|^2 + \mu h(\zbm),
\end{align*}
where $\phi^\star$ denotes the conjugate of $\phi$. The function $\phi$ is closed and $1$-strongly convex. Hence, we know that $\phi^\star$ is defined for all $\xbm \in \R^n$ and is differentiable with gradient~\cite{Boyd.Vandenberghe2004}
$$\nabla \phi^\star(\xbm) = \argmax_{\zbm \in \R^n} \left\{\xbm^\Tsf\zbm - \mu h(\zbm) - \frac{1}{2}\|\zbm\|^2\right\} = \prox_{\mu h}(\xbm).$$
Hence, we conclude that
$$\nabla h_\mu(\xbm) = \xbm - \nabla \phi^\star(\xbm) = \xbm - \prox_{\mu h}(\xbm).$$
Note that since the proximal operator is firmly nonexpansive, $\nabla h_\mu$ is also firmly nonexpansive, which means that it is $1$-Lipschitz.
\end{proof}

The next result shows that the Moreau envelope can serve as a smooth approximation to a nonsmooth function.
\begin{proposition}
\label{Prop:UniformBoundMoreau}
Consider $h \in \R^n$ and its Moreau envelope $h_\mu(\xbm)$ for $\mu > 0$. Then,
$$0 \leq h(\xbm) - \frac{1}{\mu}h_\mu(\xbm) \leq \frac{\mu}{2}G_\xbm^2\quad\text{with}\quad G_\xbm^2 \defn \min_{\gbm \in \partial h(\xbm)} \|\gbm\|^2, \quad \xbm \in \R^n.$$
\end{proposition}

\begin{proof}
First note that
$$\frac{1}{\mu}h_\mu(\xbm) = \min_{\zbm \in \R^n}\left\{\frac{1}{2\mu}\|\zbm-\xbm\|^2 + h(\zbm)\right\} \leq h(\xbm), \quad \xbm \in \R^n,$$
which is due to the fact that $\zbm = \xbm$ is potentially suboptimal. We additionally have for any $\gbm \in \partial h(\xbm)$
\begin{align*}
h_\mu(\xbm) - \mu h(\xbm) &= \min_{\zbm \in \R^n}\left\{\mu h(\zbm) - \mu h(\xbm) + \frac{1}{2}\|\zbm-\xbm\|^2\right\} \\
&\geq \min_{\zbm \in \R^n}\left\{\mu \gbm^\Tsf(\zbm-\xbm) + \frac{1}{2}\|\zbm-\xbm\|^2\right\} \\
&= \min_{\zbm \in \R^n} \left\{\frac{1}{2}\|\zbm-(\xbm-\mu\gbm)\|^2 - \frac{\mu^2}{2}\|\gbm\|^2\right\}\\
&= -\frac{\mu^2}{2}\|\gbm\|^2.
\end{align*}
This directly leads to the conclusion.
\end{proof}

\section{Additional Technical Details}
\label{Sec:TechnicalDetails}

In this section, we discuss several technical details that we omitted from the main paper for space. Section~\ref{Sec:ComputationalComplexity} discusses issues related to implementation and computational complexity of BC-RED. Section~\ref{Sec:ArchitectureTraining} discusses the architecture of our own CNN denoiser $\DnCNNast$ and provides details on its training. Section~\ref{Sec:InfluenceLipschitz} discusses the influence of the Lipschitz constant of the CNN denoiser on its performance as a denoising prior.

\subsection{Computational Complexity and a Coordinate-Friendly Implementation}
\label{Sec:ComputationalComplexity}

Theoretical analysis in Section~\ref{Sec:TheoretcalResults} of the main paper suggests that, if $b$ updates of BC-RED (each modifying a single block) are counted as a single iteration, the worst-case convergence rate of BC-RED is expected to be better than that of the full-gradient RED. This fact was empirically validated in Section~\ref{Sec:Simulations}, where we showed that in practice BC-RED needs much fewer iterations to converge. However, the overall computational complexity of two methods depends on their per-iteration complexities. In particular, the overall complexity of BC-RED is favorable when its total number of iterations required for convergence offsets the cost of solving the problem in a block-coordinate fashion. As for traditional coordinate descent methods~\cite{Peng.etal2016, Niu.etal2011}, in many problems of interest, the computational complexity of a single update of BC-RED will be roughly $b$ times lower than that of the full-gradient method.

The computational complexity of each block-update will depend on the specifics of the data-fidelity term $g$ and the denoiser $\Dsf$ used in the estimation problem. For example, consider the problem where $g(\xbm) = \frac{1}{2}\|\Abm\xbm-\ybm\|_2^2$. Additionally, suppose that $\xbm$ is such that it is sufficient represent its prior with a block-wise denoiser on each $\xbm_i$, rather than on the full $\xbm$. This situation is very common in image processing, where many popular denoisers are applied block-wise~\cite{Zoran.Weiss2011}. Then, one can obtain a very efficient implementation of BC-RED, illustrated in Algorithm~\ref{Alg:BCRED2}.

The worst-case complexity of applying $\Abm_i$ and $\Abm_i^\Tsf$ is $O(m n_i)$, which means that the cost of $b$ updates such updates for $i \in \{1, \dots, b\}$ is $O(mn)$. Additionally, if the complexity of $b$ block-wise denoising operations is equivalent or less than the complexity of denoising the full vector (which is generally true for advanced denoisers), then the complexity of $b$ updates of BC-RED will be equivalent or better than a single iteration of the full-gradient RED.

Some of our simulations were conducted using denoisers applied on the full-image and others using block-wise denoisers. In particular, the convergence simulations in Fig.~\ref{Fig:convergenceCT} and Fig.~\ref{Fig:ConvergencePlots} relied on the full-image denoisers, in order to use identical denoisers for both RED and BC-RED and be fully compatible with the theoretical analysis. On the other hand, the SNR results in Table~\ref{Tab:SNR}, Table~\ref{Tab:LipschitzDiscuss}, Fig.~\ref{Fig:MoreExamples}, and Fig.~\ref{Fig:imageFlow} rely on block-wise denoisers, where the denoiser input includes an additional 40 pixel padding around the block and the output has the exact size of the block. The padding size was determined empirically in order to have a close match between BC-RED and RED. We have observed that having even larger paddings does not influence the results of BC-RED. Finally, the size of the denoiser input and output for the galaxy simulations in Fig.~\ref{Fig:galaxyImages} and Fig.~\ref{Fig:MoreGalaxies} exactly matches the block size, with no additional padding. 

\begin{algorithm}[H]
\caption{BC-RED for the least-squares data-fidelity and a block-wise denoiser}
\label{Alg:BCRED2}
\begin{algorithmic}[1]
\STATE \textbf{input: } initial value $\xbm^0 \in \R^n$, parameter $\tau > 0$, and step-size $\gamma > 0$.
\STATE \textbf{initialize: } $\rbm^0 \leftarrow \Abm\xbm^0-\ybm$
\FOR{$k = 1, 2, 3, \dots$}
\STATE Choose an index $i_k \in \{1, \dots, b\}$
\STATE $\xbm^k \leftarrow \xbm^{k-1} - \gamma \Usf_{i_k} \Gsf_{i_k}(\xbm^{k-1})$ \quad with\quad  $\Gsf_{i_k}(\xbm^{k-1}) = \Abm_{i_k}^\Tsf \rbm^{k-1} + \tau (\xbm_{i_k} - \Dsf(\xbm_{i_k}))$.
\STATE $\rbm^k \leftarrow \rbm^{k-1} - \gamma \Abm_{i_k} \Gsf_{i_k}(\xbm^{k-1})$
\ENDFOR
\end{algorithmic}
\end{algorithm} 

\subsection{Architecture and Training of $\DnCNNast$ }
\label{Sec:ArchitectureTraining}

\begin{figure}[t]
\centering\includegraphics[width=0.45\textwidth]{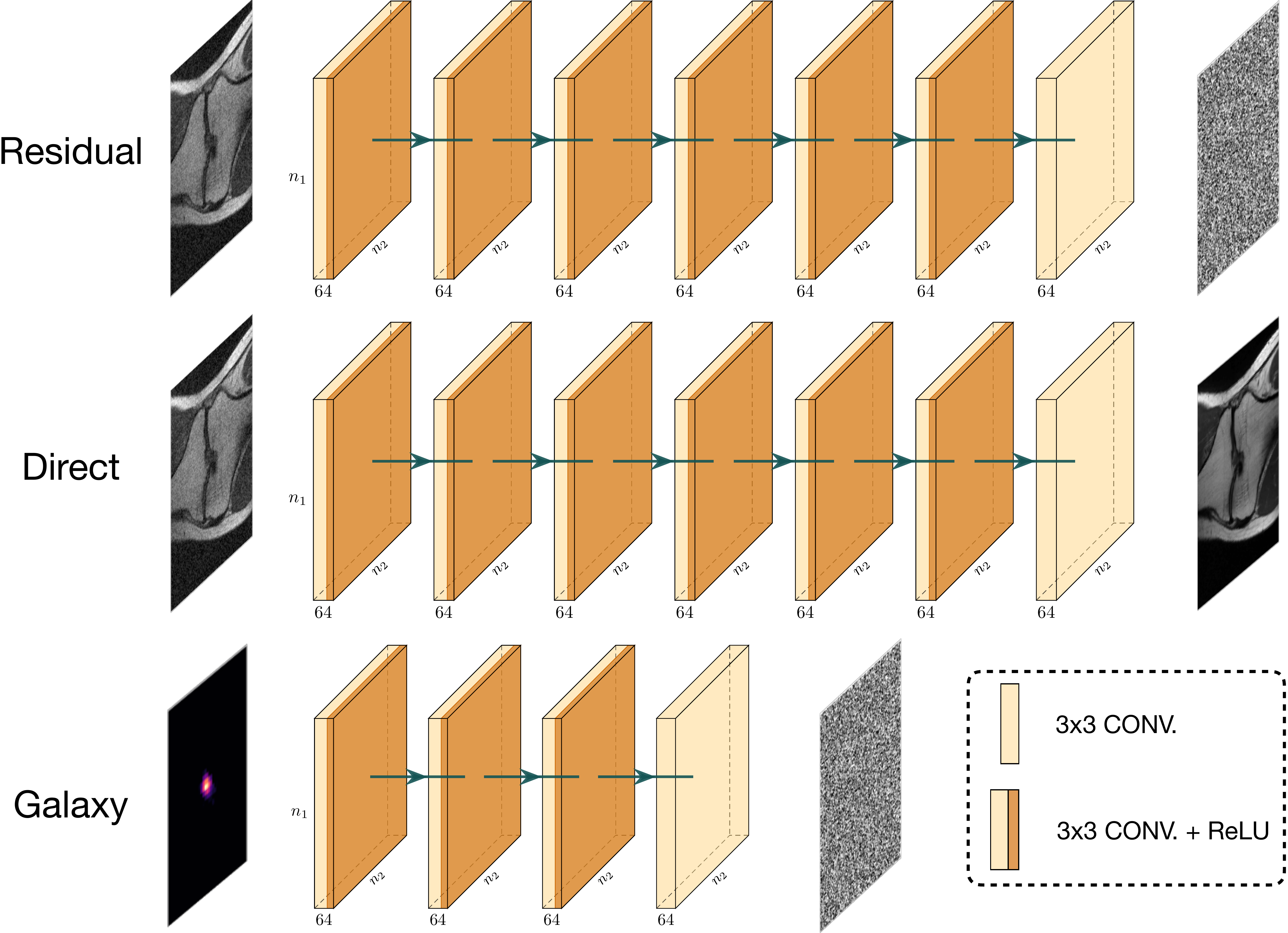}
\caption{The architecture of three variants of $\DnCNNast$ used in our simulations. Each neural net is trained to remove AWGN from noisy input images. \textbf{Residual $\DnCNNast$} is trained to predict the noise from the input. The final desired denoiser $\Dsf$ is obtained by simply subtracting the predicted noise from the input $\Dsf(\zbm) = \zbm - \mathsf{DnCNN}^\ast(\zbm)$. \textbf{Direct $\DnCNNast$} is trained to directly output a clean image from a noisy input $\Dsf(\zbm) = \mathsf{DnCNN}^\ast(\zbm)$. \textbf{Galaxy $\DnCNNast$} is a further simplification of the Residual DnCNN to only 4 convolutional layers specifically designed for large-scale image recovery. In most experiments, we further constrain the Lipschitz constant (LC) of the direct denoiser to be LC = 1 and of the residual denoiser to LC = 2 by using spectral normalization~\cite{Sedghi.etal2019}. LC = 1 means that $\Dsf$ is a nonexpansive denoiser. A residual $\Rsf = \Isf - \Dsf$ with LC = 2 provides a necessary (but not sufficient) condition for $\Dsf$ to be a nonexpansive denoiser.}
\label{Fig:DnCNNstar}
\end{figure}

We designed $\DnCNNast$ fully based on DnCNN architecture. The network contains three parts. The first part is a composite convolutional layer, consisting of a normal convolutional layer and a rectified linear units (ReLU) layer. It convolves the $n_1 \times n_2$ input to $n_1 \times n_2 \times 64$ features maps by using 64 filters of size $3 \times 3$. The second part is a sequence of 5 composite convolutional layers, each having 64 filters of size $3 \times 3 \times 64$. Those composite layers further processes the feature maps generated by the first part. The third part of the network, a single convolutional layer, generates the final output image by convolving the feature maps with a $3 \times 3 \times 64$ filter. Every convolution is performed with a stride $=1$, so that the intermediate feature maps share the same spatial size of the input image. Fig.~\ref{Fig:DnCNNstar} visualizes the architectural details. We generated 52000 training examples by adding AWGN to 13000 images ($320 \times 320$) from the NYU fastMRI dataset~\cite{Zbontar.etal2018} and cropping them into 4 sub-images of size $160 \times 160$ pixels. We trained $\DnCNNast$ to optimize the \emph{mean squared error} by using the Adam optimizer. 


\subsection{Influence of the Lipschitz Constant on Performance} 
\label{Sec:InfluenceLipschitz}

\begin{table*}[t]
\centering
\caption{Average SNR achieved by BC-RED for two variants of $\text{DnCNN}^\ast$ at different Lipschitz constant (LC) values. Note how the stability of nonexpansive (LC = 1) direct $\DnCNNast$ comes with a suboptimal SNR performance. On the other hand, the excellent SNR performance of unconstrained direct $\DnCNNast$ comes with algorithmic instability. Finally, the residual $\DnCNNast$ with LC = 2 leads to both stable convergence and nearly SNR optimal results in all our simulations.}\vspace{5pt}
\label{Tab:LipschitzDiscuss}
\begin{tabular*}{13.7cm}{C{35pt}C{55pt}C{30pt}C{30pt}cC{30pt}C{30pt}cC{30pt}C{30pt}} \toprule
\multicolumn{2}{c}{\multirow{ 2}{*}{\textbf{Variants of $\text{DnCNN}^\ast$}}} & \multicolumn{2}{c}{\textbf{Radon}} & & \multicolumn{2}{c}{\textbf{Random}} & & \multicolumn{2}{c}{$\textbf{Fourier}$} \\
\cmidrule{3-4} \cmidrule{6-7} \cmidrule{9-10}
& & \textbf{30 dB} & \textbf{40 dB} & & \textbf{30 dB} & \textbf{40 dB} & & \textbf{30 dB} & \textbf{40 dB} \\
\midrule
\multirow{2}{*}{\textbf{Direct}} & Unconstrained & 21.67 & 24.74 & & Diverges & Diverges & & 29.40 & 30.35 \\
						& LC = 1 & 19.33 & 22.98 & &19.89 & 20.26 & & 25.06 & 25.40 \\
\cmidrule{1-10}
\multirow{2}{*}{\textbf{Residual}} & Unconstrained & 20.88 & 24.68 & & 26.49 & 27.60 & & 29.39 & 30.31 \\
						& LC = 2 & 20.88 & 24.42 & & 26.60 & 28.12 & & 29.40 & 30.39 \\
\bottomrule
\end{tabular*}
\end{table*}

Our theoretical analysis in Theorem~\ref{Thm:ConvThm1} assumes that the denoiser each block denoiser $\Dsf_i$ of $\Dsf$ is block-nonexpansive. It is relatively straightforward to control the global Lipscthiz constants of CNN denoisers via spectral normalization~\cite{Miyato.etal2018, Sedghi.etal2019, Gouk.etal2018} and we have empirically tested the influence of nonexpansiveness to the quality of final image recovery.

Table~\ref{Tab:LipschitzDiscuss} summarizes the SNR performance of BC-RED for two common variants of $\DnCNNast$. The first variant is trained to learn the \emph{direct} mapping from a noisy input to a clean image, while the second variant relies on \emph{residual learning} to map its input to noise (shown in Fig.~\ref{Fig:DnCNNstar}). To gain insight into the influence of the \emph{Lipschitz constant (LC)} of a denoiser to its performance as a prior, we trained denoisers with both globally constrained and nonconstrained LCs via the spectral-normalization technique from~\cite{Sedghi.etal2019}. For the direct network, we trained $\DnCNNast$ with ${\text{LC} = 1}$, which corresponds to a nonexpansive denoiser. For the residual network, we considered $\text{LC} = 2$, which is a necessary (but not sufficient) condition for the nonexpansiveness. In our simulations, BC-RED converged for all the variants of $\DnCNNast$, except for the direct and unconstrained $\DnCNNast$, which confirms that our theoretical analysis provides only sufficient conditions for convergence. Nonetheless, our simulations reveal the performance loss of the algorithm for the direct and nonexpansive (LC = 1) $\DnCNNast$. On the other hand, the performance of the residual $\DnCNNast$ with $\text{LC} = 2$ nearly matches the performance of fully unconstrained networks in all experiments. 

\section{Additional Numerical Validation}
\label{Sec:AdditionalSimulations}

\begin{figure*}[t]
\centering\includegraphics[width=0.9\textwidth]{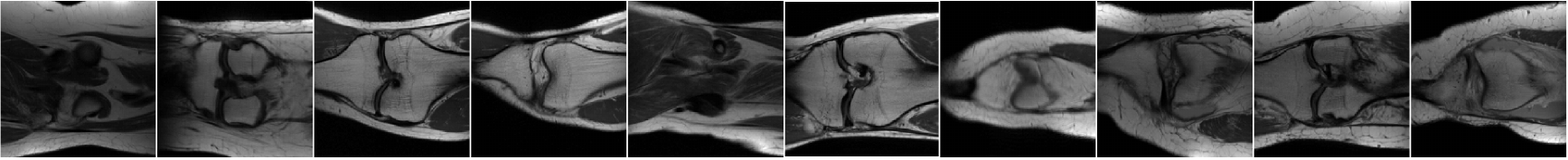}
\caption{Ten randomly selected test images from the fastMRI knee dataset~\cite{Zbontar.etal2018}.}
\label{Fig:TestImages}
\end{figure*}

Fig.~\ref{Fig:TestImages} shows ten randomly selected  test images used for numerical validation. The simulations in this paper were performed on a machine equipped with an Intel Xeon Gold 6130 Processor that has 16 cores of 2.1 GHz and 192 GBs of DDR memory. We trained all neural nets using NVIDIA RTX 2080 GPUs. 

Fig.~\ref{Fig:ConvergencePlots} presents the convergence plots for \emph{direct} and \emph{residual} $\DnCNNast$ with Radon matrix. In order to ensure nonexpansivenss, the LC of direct $\DnCNNast$ is constrained to 1. On the other hand, the LC of the residual $\DnCNNast$ is constrained to 2, which is a necessary condition for ensuring its nonexpansiveness. We compare two variants of BC-RED, one with \emph{i.i.d.}~block selection and an alternative that proceeds in \emph{epochs} of $b$ consecutive iterations, where at the start of each epoch the set $\{1, \dots, b\}$ is reshuffled, and $i_k$ is then selected consecutively from this ordered set. The figure first confirms our observation of the convergence of BC-RED under different $\DnCNNast$, and further highlights the faster convergence speed of BC-RED due to its ability to select larger step-size and immediately reuse each block update. Among two block selection rules, \emph{BC-RED (epoch)} clearly outperforms \emph{BC-RED (i.i.d.)} in all our simulations, which has also been observed in traditional coordinate descent methods~\cite{Wright2015}. However, the theoretical understanding of this gap in performance between \emph{epoch} and \emph{i.i.d.} block selection remains elusive.

Fig.~\ref{Fig:MoreExamples} visually compares the images recovered by BC-RED and RED and two baseline methods. First, the images visually illustrate the excellent agreement between BC-RED and RED. Second, leveraging advanced denoisers in BC-RED largely improves the reconstruction quality over PGM with the traditional TV prior. For instance, BC-RED under $\DnCNNast$ outperforms PGM under TV by 1 dB for Fourier matrix. Finally, we note the stability of BC-RED using the CNN denoiser versus the deteriorating performance of U-Net, which is trained end-to-end for Radon matrix with 30 dB noise. This fact highlights one key merit of the RED framework, that the CNN denoiser, only trained once, can be directly applied in different scenarios for different tasks with no degradation.

\begin{figure*}[t]
\centering\includegraphics[width=0.8\textwidth]{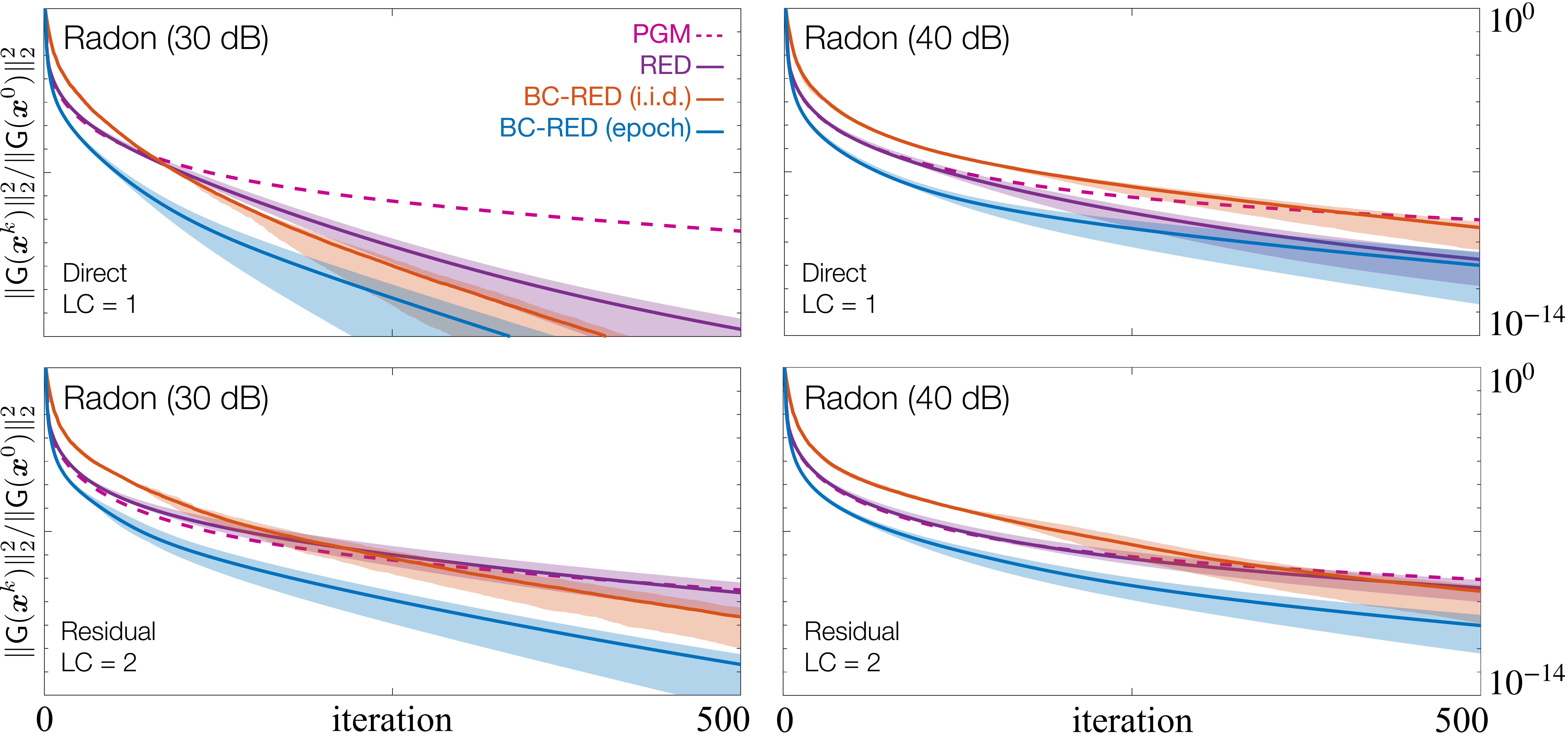}
\caption{\textbf{Left column} shows the convergence of BC-RED under different $\DnCNNast$ priors for Radon matrix with 30 dB noise. The top figure corresponds to the nonexpansive, direct $\DnCNNast$, while the bottom figure corresponds to the residual $\DnCNNast$ with Lipschitz constant of two. \textbf{Right column} shows the convergence of BC-RED under the same set of $\DnCNNast$ priors for Radon matrix with 40 dB noise. Average normalized distance to $\zer(\Gsf)$ is plotted against the iteration number with the shaded area representing the range of values taken over all test images. We observed general stability of BC-RED across all simulations for direct $\DnCNNast$ with LC = 1 and residual $\DnCNNast$ with {LC = 2.}}
\label{Fig:ConvergencePlots}
\end{figure*}

\begin{figure*}[t]
\centering\includegraphics[width=0.8\textwidth]{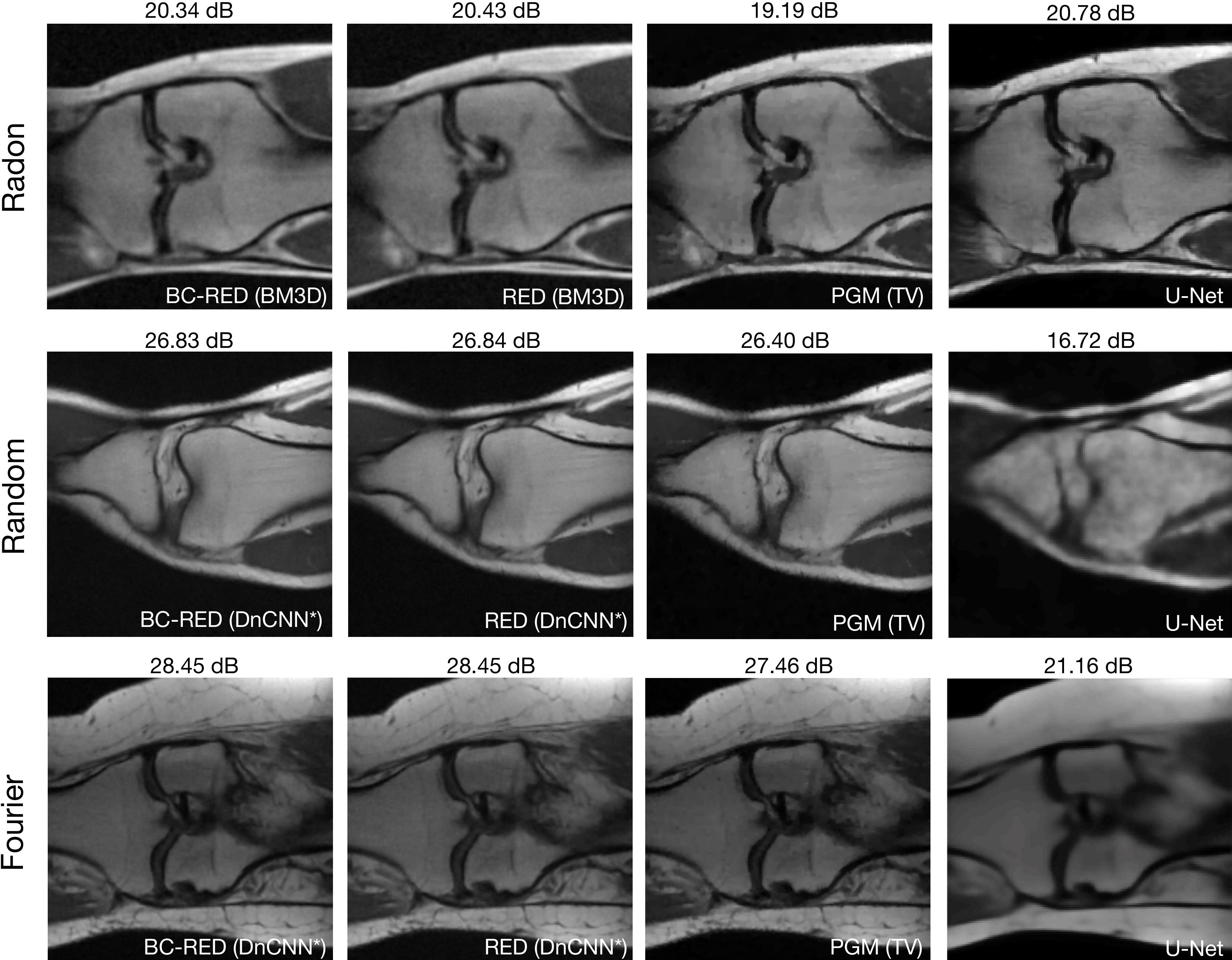}
\caption{Visual comparison between BC-RED and RED against PGM (TV) and U-Net for all three matrices with 30 dB noise. For BC-RED and RED, we selected the denoiser resulting in the best reconstruction performance. Every image is marked by its SNR value with respect to the ground truth. We highlight the excellent agreement between BC-RED and RED in all experiments. Note the strong degradation in the image quality for U-Net, due to the mismatch between the training and testing.}
\label{Fig:MoreExamples}
\end{figure*}

In BC-RED, the parameter $\tau$ controls the tradeoff between $\zer(\nabla g)$ and $\fix(\Dsf)$. Fig.~\ref{Fig:imageFlow} illustrates evolution of images reconstructed by BC-RED for different $\tau$. The first row corresponds to the reconstruction from the Fourier measurements with 30 dB noise, while the second row corresponds to the Radon measurements with 40 dB noise. The figure clearly shows how $\tau$ explicitly adjusts the balance between the data-fit and the denoiser. In particular, small $\tau$, corresponding to weak denoising, results in unwanted artifacts in the reconstructed images, while large $\tau$ promotes denoising strength but smooths out desired features and details. The leftmost images in Fig.~\ref{Fig:imageFlow}  shows the optimal balance introduced by $\tau^\ast$.

To conclude, we present the experimental details of the galaxy image recovery task. In the simulation, we inherited the dataset used in \cite{Farrens.etal2017}. The dataset\footnotemark~contains 10'000 galaxy survey images from the GREAT3 Challenge \cite{Mandelbaum.etal2014}, and each image is cropped to $41 \times 41$ pixel size. The dataset also includes 597 simulated space variant point spread functions (PSF) corresponding to 597 physical position across 4 $4096 \times 4132$ pixel CCDs~\cite{Cropper.etal2012, Kuntzer.etal2016}. In order to synthesize the $8292 \times 8364$ pixel image, we first selected 597 galaxy images from the dataset and degraded each of them by a different PSF, and then locate the degraded images back to the corresponding positions in the full image. Note that we also contaminated each degraded image with AWGN of 5 dB. Figure~\ref{Fig:DnCNNstar} shows the architecture of the 4-layer $\DnCNNast$ used as denoiser for the galaxy image recovery. We generated 72000 training examples by rotating and flipping the rest 9000 images, and trained the neural network to learn the noise residual with LC$=2$.

\footnotetext{http://www.cosmostat.org/deconvolution}

Since the locations of galaxies were known in this case, we optimized the speed of BC-RED by only updating the blocks containing galaxies. In practice, such block selection strategies can be efficiently implemented by applying a threshold on image intensities to separate blocks with galaxies from the ones that have only noise. As illustrated in Fig.~\ref{Fig:ConvergenceGalaxy}, BC-RED converged to about $4.78 \times 10^{-5}$, in relative accuracy within 120 seconds, which corresponds to 100 iterations of the algorithm, with $b$ BC-RED updates grouped as a single iteration. Fig.~\ref{Fig:MoreGalaxies} illustrates the performance of BC-RED under $\DnCNNast$ for 4 example galaxies selected from the $1316 \times 1245$ pixel sub-image. The first row on the left shows the same galaxy in Fig.~\ref{Fig:galaxyImages} in the main paper. We obtained the reconstructed image of the low-rank matrix prior by running the algorithm with default parameter values. This experiment demonstrates that BC-RED can indeed be applied to a realistic, nontrivial image recovery task on a large image.


\begin{figure*}[t]
\centering\includegraphics[width=0.8\textwidth]{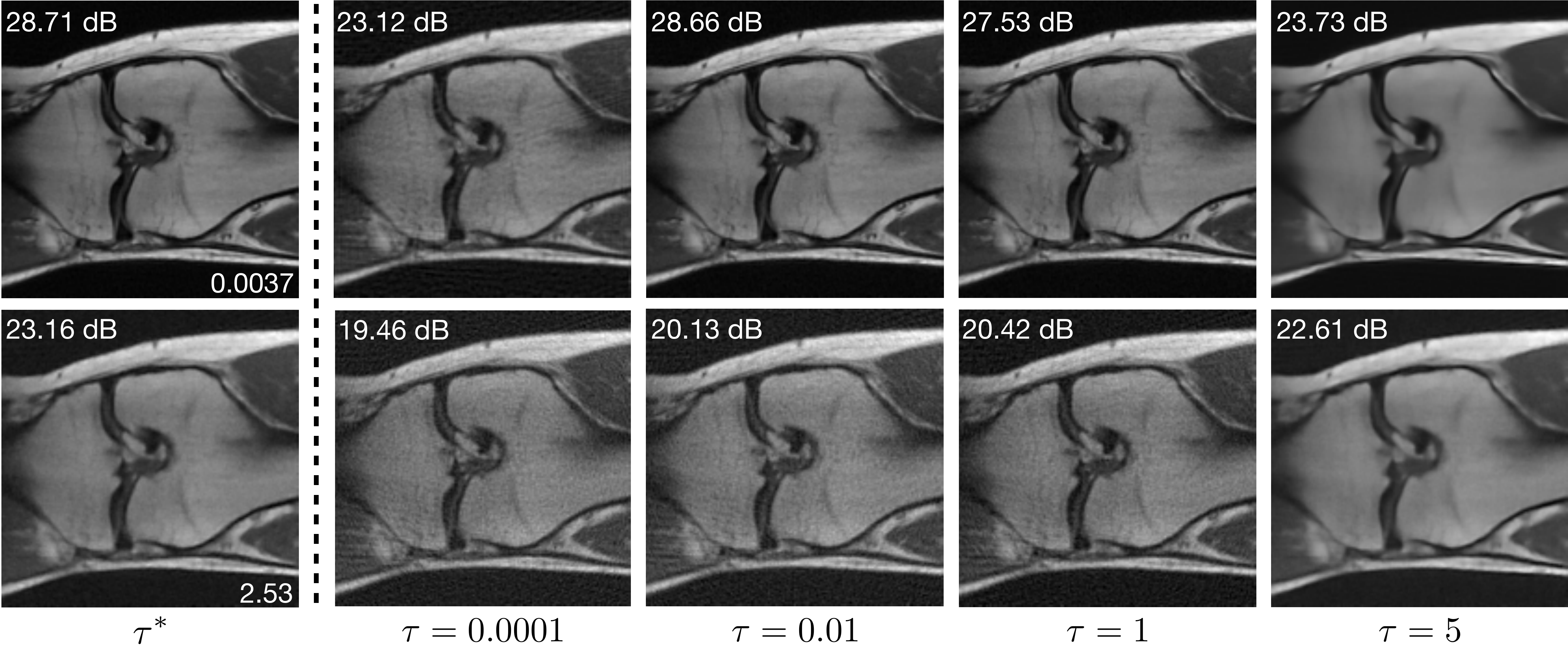}
\caption{Evolution of the images reconstructed by BC-RED using the $\DnCNNast$ denoiser for different values of $\tau$. The first row corresponds to Fourier matrix with 30 dB noise, while the second row corresponds to the Radon matrix with 40 dB noise. Each reconstructed image is marked with its SNR value with respect to the ground truth image. The optimal parameters $\tau^\ast$ for the two problems are $0.0037$ and $2.35$, respectively. The denoiser used in this simulation is the residual $\DnCNNast$ with a Lipschitz constant LC = 2. This figure illustrates how $\tau$ enables an explicit tradeoff between the data-fit and the regularization.}
\label{Fig:imageFlow}
\end{figure*}

\begin{figure}[t]
\centering\includegraphics[width=0.45\textwidth]{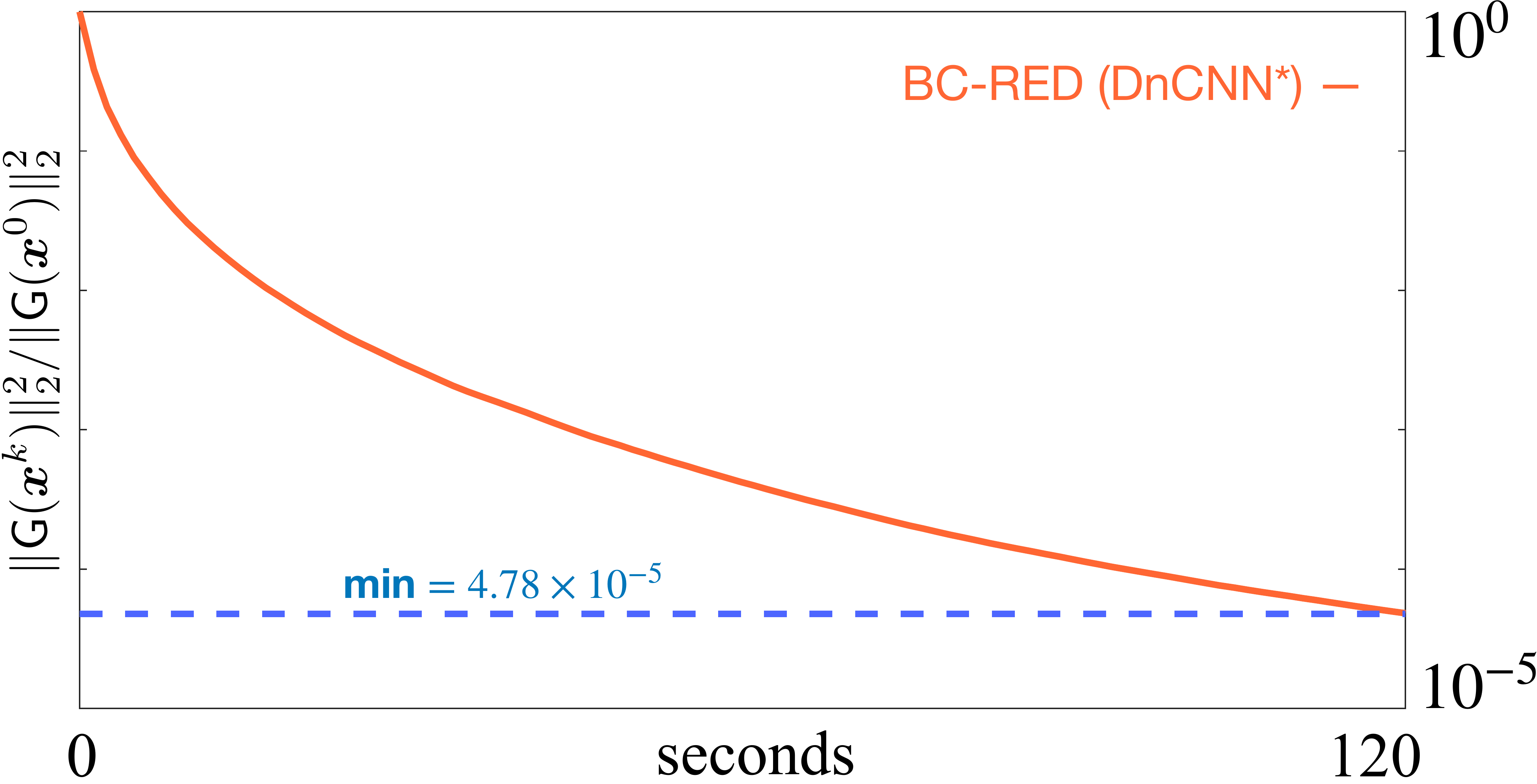}
\caption{Illustration of the convergence of BC-RED under $\DnCNNast$ in the realistic, large-scale image recovery task. BC-RED is run for 100 iterations, which leads to the accuracy of $4.78 \times 10^{-5}$ within 120 seconds. The efficiency of the algorithm is due to the sparsity of the recovery problem.}
\label{Fig:ConvergenceGalaxy}
\end{figure}

\begin{figure*}[t]
\centering\includegraphics[width=0.8\textwidth]{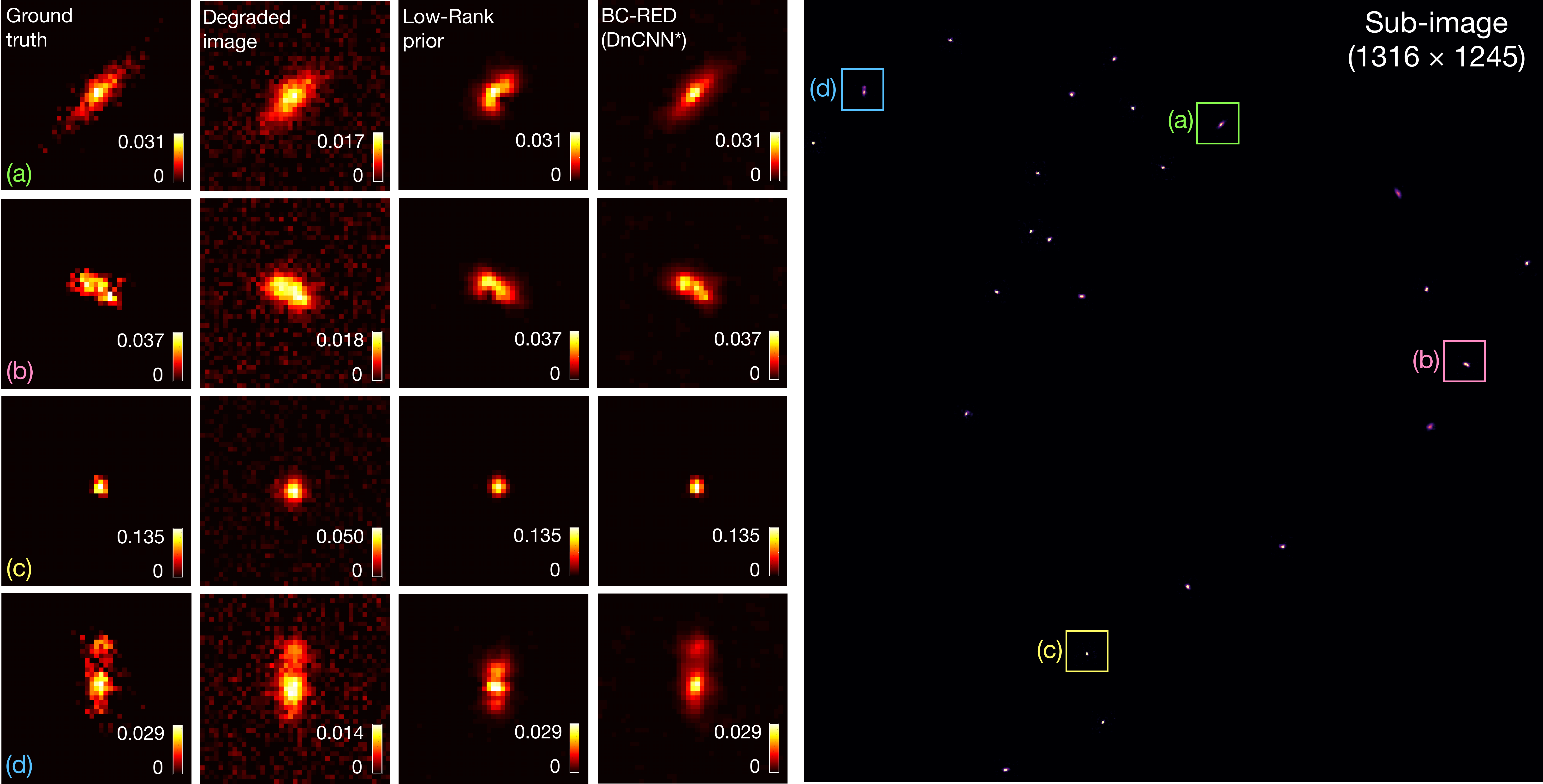}
\caption{Illustration of performance of BC-RED under residual $\DnCNNast$ denoiser with LC = 2. The first and the second columns show the ground truth images and the blocks from the measurement, respectively. The third and the forth columns are the reconstructed results obtained by BC-RED and the low-rank matrix prior \cite{Farrens.etal2017}, respectively. The rightmost image is a $1316\times1245$ pixel sub-image of the full-sized $8292 \times 8364$ pixel reconstructed image obtained by BC-RED. Note that the intent of this figure is not to justify $\DnCNNast$ as a prior for image recovery, but to demonstrate that BC-RED can indeed be applied to a realistic, nontrivial image recovery task on a large image.}
\label{Fig:MoreGalaxies}
\end{figure*}

\bibliographystyle{IEEEtran}


\end{document}